\documentclass{article}

\PassOptionsToPackage{numbers, compress}{natbib}

\usepackage{microtype}
\usepackage{graphicx}
\usepackage{subfigure}
\usepackage{booktabs} 
\usepackage{hyperref}

\usepackage[preprint]{neurips_2023}

\usepackage{amsmath}
\usepackage{amssymb}
\usepackage{mathtools}
\usepackage{amsthm}

\usepackage{enumitem}

\usepackage[capitalize,noabbrev]{cleveref}

\theoremstyle{plain}
\newtheorem{theorem}{Theorem}
\newtheorem{proposition}{Proposition}
\newtheorem{lemma}{Lemma}

\newtheorem{example}{Example}
\newtheorem{remark}{Remark}

\usepackage[textsize=tiny]{todonotes}
\usepackage{mystyle}

\usepackage[utf8]{inputenc} 
\usepackage[T1]{fontenc}    
\usepackage{hyperref}       
\usepackage{url}            
\usepackage{booktabs}       
\usepackage{amsfonts}       
\usepackage{nicefrac}       
\usepackage{microtype}      
\usepackage{xcolor}         

\title{Linear Bandits with Memory: from Rotting to Rising}

\author{%
  Giulia Clerici\\
  Department of Computer Science\\
  Università degli Studi di Milano\\
  Milan, Italy \\
  \texttt{giulia.clerici@unimi.it} \\
  \And
  Pierre Laforgue \\
  Department of Computer Science\\
  Università degli Studi di Milano\\
  Milan, Italy \\
  \texttt{pierre.laforgue@unimi.it} \\
  \And
  Nicolò Cesa-Bianchi \\
  Department of Computer Science\\
  Università degli Studi di Milano\\
  Milan, Italy \\
  \texttt{nicolo.cesa-bianchi@unimi.it} \\
}

\begin{document}

\maketitle

\begin{abstract}
%
Nonstationary phenomena, such as satiation effects in recommendations, have mostly been modeled using bandits with finitely many arms.
However, the richer action space provided by linear bandits is often preferred in practice.
In this work, we introduce a novel nonstationary linear bandit model, where current rewards are influenced by the learner's past actions in a fixed-size window.
Our model, which recovers stationary linear bandits as a special case, leverages two parameters: the window size $m \ge 0$, and an exponent $\gamma$ that captures the rotting ($\gamma < 0)$ or rising ($\gamma > 0$) nature of the phenomenon.
When both $m$ and $\gamma$ are known, we propose and analyze a variant of OFUL which minimizes regret against cycling policies.
By choosing the cycle length so as to trade-off approximation and estimation errors, we then prove a bound of order $\sqrt{d}\,(m+1)^{\frac{1}{2}+\max\{\gamma,0\}}\,T^{3/4}$ (ignoring log factors) on the regret against the optimal sequence of actions, where $T$ is the horizon and $d$ is the dimension of the linear action space.
Through a bandit model selection approach, our results are extended to the case where $m$ and $\gamma$ are unknown.
Finally, we complement our theoretical results with experiments against natural baselines.
\end{abstract}

\section{Introduction}
\label{sec:intro}

Many real-world problems are naturally modeled by stochastic linear bandits, where actions belong to a linear space, and the learner obtains rewards whose expectations are linear functions of the chosen action  (see, e.g., \citep{lattimore2020bandit}).
Formally, at each time step $t$ the expected reward is $r_t = \langle a_t, \theta^*\rangle$, where
$a_t \in \mathbb{R}^d$ is the chosen action and $\theta^* \in \mathbb{R}^d$ is a fixed and unknown parameter to be estimated.
In a song recommendation problem, the possible actions are the songs from the catalogue, seen as vectors in the linear space defined by the songs' features, while the linear reward $r_t$ (i.e., the user's satisfaction) measures how well the song $a_t$ picked by the learner matches the (unknown) preferences of the user, represented by $\theta^*$.
However, this model fails to capture a key aspect, i.e., the nonstationarity of the users' preferences.
For example, user satiation with respect to the recommended items is a typical phenomenon in this context \citep{kapoor2015just,kunaver2017diversity}.
Indeed, identifying the favorite song of a user (i.e., the vector $a$ in the action set that maximizes $\langle a, \theta^*\rangle$) only partly solves the recommendation problem, as suggesting this song repeatedly is not meaningful in the long run \citep{kovacs2018rotating,schedl2018current}.
Whereas satiation phenomena are typical in recommendation settings, a different kind of nonstationarity may arise in other domains.
In algorithmic selection for instance, one must choose among a pool of algorithms the one that is going to get the next chunk of resources (e.g., CPU time or samples).
In this case, we expect the quality of the solution found by each algorithm to increase (as opposed to decrease) as the algorithm gets selected.
This model, known as rising bandits, has been studied in deterministic \citep{heidari2016tight,li2020efficient} and stochastic \citep{metelli2022stochastic} settings.\looseness-1

Nonstationarity in bandits, which has been mostly studied in the case of finitely many arms, appears to be significantly more intricate to analyze in a linear bandit framework due to the structure of the action space.
For instance, rotting bandits \citep{bouneffouf2016multi} or rested rising bandits \citep{metelli2022stochastic} assume that the expected reward of an arm is fully determined by the number of times this arm has been pulled in the past.
In the linear case, on the contrary, one would expect nontrivial cross-arm effects.
Listening to rock songs should affect the future interest in rock songs, but also to a minor extent that in pop music, as the two genres are related.
In addition, most songs cannot be described by a single genre.
It then seems reasonable that a pop-rock song does not increase rock satiation as much as a pure rock song.
Hence, new ways of modeling nonstationary phenomena, such rotting and rising, are required in linear environments.\looseness-1

In this work, we introduce a novel linear bandit framework that allows to model complex nonstationary behaviors in an infinite and structured space of actions.
More specifically, the nonstationarity is captured by a matrix, determined by the past actions of the learner and affecting the expected reward of future actions.
Formally, the expected reward at time step $t$ becomes $r_t = \langle a_t, A_{t-1} \theta^* \rangle$, where $A_{t-1} = A(a_{t-1}, \ldots, a_{t-m}) = \left(A_0 + \sum_{s=1}^m a_{t-s}a_{t-s}^\top\right)^\gamma \in \mathbb{R}^{d \times d}$.
Here, $A_0$ is some initial symmetric and positive semidefinite matrix. Typically, $A_0$ is chosen to be the identity $I_d$, which we refer to as the isotropic initialization.
The memory size $m \ge 0$ controls the range of past actions having an influence, while the exponent $\gamma\in\mathbb{R}$ quantifies their impact.
A positive $\gamma$ corresponds to rising and a negative one to rotting.
In the rotting setting, playing action $a$ at time $t$ decreases the expected reward of $a$ at time $t+1$.
Hence, solving this problem may require long-term planning and playing repeatedly $\theta^*$ may not be optimal.
Instead, in a rising scenario with isotropic initialization, an optimal action at time $t$, if played, remains optimal at time $t+1$ since it has been boosted by the previous play.
Although optimal policies are stationary in this case, such problems are difficult because the learner is penalized twice: for not choosing a good action at the current time step, but also at future time steps, for not having boosted the right action.
We highlight that our approach is able to cope with these two very different scenarios.
Finally, note that our model recovers stationary linear bandits as a special case when $\gamma=0$ (or, equivalently, when $m=0$ and $A_0 = I_d$).

We start by focusing on cyclic policies, and show that they provide a reasonable approximation to the optimal policy (which may not be cyclic) while being easier to learn.
When $m$ and $\gamma$ are known, estimating the best block of fixed length reduces to a stationary problem, that we solve using a block variant of OFUL~\citep{abbasi2011improved}.
%
When $m=0$, our variant recovers the regret bound $\mathcal{O}\big(d\sqrt{T}\big)$ of OFUL up to log factors.
We then optimize the block length in order to balance the approximation and estimation errors, and obtain a bound on the regret against the optimal sequence of actions in hindsight of order $\sqrt{d}\,(m+1)^{\frac{1}{2}+\max\{\gamma,0\}}T^{3/4}$ (ignoring log factors) for all $T \ge (md)^2$. Note that the best known general lower bound for our setting is $\Omega\big(d\sqrt{T}\big)$. As we show that the approximation error is not improvable, our upper bound could be tightened by either improving on the analysis of the estimation error, or by using a more direct approach to learn the optimal strategy.
%
Finally, we extend our analysis to the case when $m$ and $\gamma$ are both unknown. For this case, we prove regret bounds via an extension of the bandit model selection approach of \cite{cutkosky2020upper}.
Empirically, our approach is shown to outperform natural baselines, such as the oracle greedy strategy (playing the action with the best instantaneous expected reward) and a naive block learning approach.
Our experimental results also include misspecified settings, where we learn $\theta^*$ and simultaneously either $m$ or $\gamma$.

\par{\bf Contributions.}
\begin{itemize}[topsep=0pt,parsep=0pt,itemsep=3pt]
\item We introduce a new bandit framework to model nonstationary effects in linear action spaces. Our model generalizes stationary linear bandits, whose bound we recover as a special case.
\item We propose an OFUL-based algorithm achieving sublinear regret against the best sequence of actions by learning cyclic policies and balancing estimation and approximation errors.
\item We use a bandit model selection approach to learn the system's parameters $m$ and $\gamma$.
\item Empirically, our algorithm outperforms natural baselines in both rotting and rising settings.
\end{itemize}

\par{\bf Related works.}
Stochastic linear bandits, which were introduced two decades ago \citep{abe1999associative,auer2002using}, are typically addressed using algorithms based on ellipsoidal confidence sets \citep{dani2008stochastic,rusmevichientong2010linearly,abbasi2011improved}.
Nonstationary bandits have been mainly studied in the case of finitely many arms.
Among the most studied models, there are
rested \citep{gittins1979bandit,gittins2011multi} and restless \citep{whittle1988restless,ortner2012regret,tekin2012online} bandits,
rotting bandits \citep{bouneffouf2016multi,heidari2016tight,cortes2017discrepancy,levine2017rotting,seznec2019rotting},
bandits with rewards \mbox{depending} on arm delays \citep{kleinberg2018recharging,pike2019recovering,cella2020stochastic,simchi2021dynamic,laforgue2022last},
blocking and rebounding bandits \citep{basu2019blocking,leqi2020rebounding},
and rising bandits \citep{li2020efficient,metelli2022stochastic}.
Some works have also considered nonstationary bandit frameworks, where the unknown parameter $\theta^*$ is then replaced by a sequence of vectors $\theta^*_t$ that evolves over time.
Standard assumptions then stipulate that $\theta^*_t$ is piecewise stationary, with a fixed number of change points \citep{bouneffouf2017context,wu2018learning,auer2019adaptively,chen2019new,di2020linear,xu2020contextual,li2021unifying}, or that the variation budget $\sum_{t\le T} \|\theta^*_t - \theta^*_{t-1}\|$ is bounded \citep{besbes2014stochastic,karnin2016multi,luo2018efficient,cheung2019learning,russac2019weighted,russac2020algorithms,kim2020randomized,zhao2020simple}. See also \citep{mueller2019low} for an application of linear bandits to nonstationary dynamic pricing.
In addition to these assumptions, we highlight that the above works are fundamentally different from ours, as the evolution of $\theta^*_t$ is oblivious to the actions taken by the learner.
This removes any need for long-term planning and puts the focus on the dynamic regret, where the algorithm's performance is compared to the rewards which one could obtain by picking $a_t$ according to $\theta^*_t$. 
%
%
Finally, note that nonstationary linear bandits may be also tackled using Gaussian Processes \citep{faury2021technical,deng2022weighted}.

\par{\bf Notation.}
$\ball$ denotes the Euclidean unit ball, $0_d$ and $(e_k)_{k \le d}$ the zero and standard basis in $\mathbb{R}^d$, $I_d \in \mathbb{R}^{d \times d}$ the identity matrix, $\|M\|_*$ the operator norm of $M$, and $\gamma^+ = \max(\gamma, 0)$ for any $\gamma \in \mathbb{R}$.
Bold characters refer to block objects, and $\tilde{\mathcal{O}}$ is used when neglecting logarithmic factors.
\looseness-1

\section{Model}
\label{sec:model}

In this section, we introduce our model of \emph{linear bandits with memory} (LBM in short).
LBMs strictly generalize stationary linear bandits, and also recover some nonstationary bandit models with finitely many arms as special cases.
As in (stochastic) linear bandits, we assume that at each time step $t = 1, 2, \ldots$ the learner picks an action $a_t$ from a (possibly infinite) set of actions $\mathcal{A} \subset \ball$, and receives a stochastic reward $y_t$.
In contrast to stationary models, however, the (expectation of the) reward is also influenced by the previous actions of the learner.
Namely, we assume the existence of an unknown vector $\theta^* \in \ball$, a memory size $m \in \mathbb{N}$, and an exponent $\gamma$ such that
\begin{equation}\label{eq:reward}
y_t = \big\langle a_t, A(a_{t-m}, \ldots, a_{t-1})\,\theta^* \big\rangle + \eta_t\,,
\end{equation}
where $\eta_t$ is a $1$-sub-Gaussian random variable independent of the actions of the learner, and
\begin{equation}\label{eq:A-def}
A(a_1, \ldots, a_m) = \bigg(A_0 + \sum_{s=1}^m a_sa_s^\top\bigg)^\gamma\,.
\end{equation}
In words, the matrix $A$ encodes how the expected reward is influenced by past actions.
The memory size $m$ tells how far in the past this influence extends, while the exponent $\gamma$ governs the type (rising or rotting) and strength of the system.
For simplicity, in the rest of the paper we use the abbreviation $A_{t-1} = A(a_{t-m}, \ldots, a_{t-1})$ and refer to it as the \emph{memory matrix}.
Conventionally, we set $a_{1-m} = a_{2-m} = \ldots = a_0 = 0_d$ and choose $A_0 = I_d$ unless otherwise stated.
Note that at any time step $t$ the expected reward $r_t = \mathbb{E}[y_t]$ satisfies $|r_t| \,\le \|A_{t-1}\|_*$.
Given a horizon $T \in \mathbb{N}$, the learner aims at maximizing the expected sum of rewards obtained over the $T$ interaction rounds.
The performance is measured against the best sequence of actions over the $T$ rounds, i.e., through the regret
\[
\sum\nolimits_{t=1}^T r^*_t - \mathbb{E}\left[\sum\nolimits_{t=1}^T y_t\right]\,,
\]
where $r^*_t = \big\langle a^*_t, A(a^*_{t-m}, \ldots, a^*_{t-1})\,\theta^* \big\rangle$ and $(a^*_t)_{t \ge 1}$ is the optimal sequence of actions, i.e., the sequence maximizing the expected sum of rewards obtained over the horizon $T$
\begin{equation}
\label{eq:opt_sequence}
a^*_1, \ldots, a^*_T = \!\argmax_{a_1,\ldots,a_T \in \mathcal{A}} \sum_{t=1}^T \big\langle a_t, A(a_{t-m}, \ldots, a_{t-1})\,\theta^* \big\rangle\,.
\end{equation}
Throughout the paper, we use $\OPT$ to denote $\sum_t r^*_t$ when the horizon $T$ is understood from the context.
Note that a LBM is fully characterized by: action set $\mathcal{A}$, parameter $\theta^*$, memory size $m$, and exponent $\gamma$. 
As shown in the following examples, besides generalizing linear bandits to a nonstationary setting, LBMs also include certain models of rotting and rising bandits with $K$ arms.

\begin{example}[Stationary linear bandits]
Consider a linear bandit model, defined by an action set $\mathcal{A} \subset \ball$ and $\theta^* \in \ball$.
This is equivalent to a LBM with the same $\mathcal{A}$ and $\theta^*$, and memory matrix $A$ such that $A(a_1, \ldots, a_m) = I_d~$ for any $a_1, \ldots, a_m \in \A^m$, i.e., when $m=0$ or $\gamma=0$.
\end{example}


\begin{example}[Rotting and rising bandits]\label{ex:rested}
In rotting \cite{levine2017rotting,seznec2019rotting} or rising \cite{metelli2022stochastic} bandits, the expected reward of an arm $k$ at time step $t$ is fully determined by the number $n_k(t)$ of times arm $k$ has been played before time $t$.
Formally, each arm is equipped with a function $\mu_k$ such that the expected reward at time $t$ is given by $\mu_k(n_k(t))$.
In particular, requiring all the $\mu_k$ to be nonincreasing corresponds to the rotting bandits model, and requiring all the $\mu_k$ to be nondecreasing corresponds to the rested rising bandits model.
Now, let $d=K$, $\mathcal{A} = (e_k)_{1 \le k \le K}$, $\theta^* = (1/\sqrt{K}, \ldots, 1/\sqrt{K})$, and $m=+\infty$.
By the definition of~$A$, see \eqref{eq:A-def}, and the orthogonality of the actions, it is easy to check that the expected reward of playing action $e_k$ at time step $t$ is given by $(1 + n_k(t))^\gamma/\sqrt{K}$.
When $\gamma \le 0$, this is a nonincreasing function of $n_k(t)$, and we recover rotting bandits.
Conversely, when $\gamma \ge 0$, we recover rising bandits.
We note however that the class of decreasing (respectively increasing) functions we can consider is restricted to the set of monomials of the form $n \mapsto (1+n)^\gamma/\sqrt{K}$, for $\gamma \le 0$ (respectively $\gamma \ge 0$).
Extending it to generic polynomials is clearly possible, although it requires learning the exponents.
\end{example}

A naive approach to learning LBM is to neglect nonstationarity.
Assuming that $\theta^*$ is known, one may then play at time $t$ the action $a^\mathrm{greedy}_t = \argmax_{a \in \mathcal{A}} \langle a, A_{t-1}\theta^*\rangle$.
Although this strategy, that we refer to as \emph{oracle greedy}, may be optimal in some cases (e.g., in rising isotropic settings, see \citet[Section~3.1]{heidari2016tight} and \citet[Theorem~4.1]{metelli2022stochastic} for discussions in the $K$-armed case), we highlight that it may also be arbitrarily bad, as stated in the next proposition (all missing proofs are found in the Supplementary Material).

\begin{restatable}{proposition}{propgreedybad}\label{prop:greedy_bad}
The oracle greedy strategy, which plays $a^\mathrm{greedy}_t = \argmax_{a \in \mathcal{A}} \langle a, A_{t-1}\theta^*\rangle$ at time $t$, can suffer linear regret, both in rotting or rising scenarios.
\end{restatable}

Hence, one must resort to more sophisticated strategies, which may include long-term planning.
Before describing our approach in the next section, we conclude the model exposition by highlighting that LBMs may also be generalized to contextual bandits \citep{lattimore2020bandit}.
\begin{remark}[Contextual bandits]\label{rmk:context}
In contextual bandits, at each time step $t$ the learner is provided a context $c_t$ (e.g., data about a user).
The learner then picks an action $a_t \in \mathcal{A}$ (based on $c_t$), and receives a reward whose expectation depends linearly on the vector $\psi(c_t, a_t) \in \mathbb{R}^d$, where $\psi$ is a known feature map.
Note that it is equivalent to have the learner playing actions $a_t \in \mathbb{R}^d$ that belong to a subset $\mathcal{A}_t = \{\psi(c_t, a) \in \mathbb{R}^d \colon a \in \mathcal{A}\}$.
The analysis developed in \Cref{sec:regret} still holds true when $\mathcal{A}_t$ depends on $t$, and can thus be generalized to contextual bandits with memory.
\end{remark}

\section{Regret Analysis}
\label{sec:regret}

In this section, we introduce and analyze OFUL-memory (\Cref{alg:OFUL-memory-algo}) for learning LBMs.
We first observe that for every block length there exists a cyclic policy providing a reasonable approximation to the optimal policy (\Cref{prop:approx_cyclic}) that cannot be improved in general (\Cref{prop:cyclic_tight}).
Learning the optimal block in the cyclic policy then reduces to a stationary linear bandit problem that can be solved by running the OFUL algorithm (\Cref{prop:adaptation_exp}).
This approach is however wasteful, as it estimates a concatenated model whose dimension scales with the block length.
We thus propose a refined algorithm leveraging the structure of the concatenated model, and show that it enjoys a better regret bound.
We then tune the block length to trade-off estimation and approximation errors (\Cref{thm:regret_exp}).
Since the optimal block length depends on the memory size $m$, which may be unknown in practice, we finally wrap our algorithm with a bandit model selection algorithm that is shown to preserve regret guarantees (\Cref{cor:model_selection}).
Throughout the analysis, we assume for simplicity that horizon $T$ is always divisible by the block length considered.
Finally, note that regret bounds are stated in expectation in the main body, while the more general high probability bounds are proved in the Supplementary Material.\looseness-1


\subsection{Approximation}
\label{sec:approx}
In LBMs, finding a block of actions maximizing the sum of expected rewards is not a well-defined problem.
Indeed, the rewards also depend on the initial conditions, determined by the $m$ actions preceding the current block.
To bypass this issue, we introduce the following proxy reward function.
For any $m,L \ge 1$ and any block $\ba = a_1\,\ldots\,a_{m+L}$ of $m+L$ actions, let
\begin{equation}\label{eq:rwd_lower_bound}
\br(\ba) = \sum_{t=m+1}^{m+L} \big\langle a_t, A_{t-1}\theta^*\big\rangle = \sum_{t=m+1}^{m+L} \big\langle A_{t-1} a_t, \theta^*\big\rangle\,.
\end{equation}
In words, we only consider the expected rewards obtained from the index $m+1$ onward.
Note that actions $a_1\,\ldots\,a_m$ still do play a role in $\br$, as they influence $A_m, \ldots, A_{2m-1}$.
The key is that $\br$ is now independent from the initial state, so that
\begin{equation}\label{eq:best_block}
\tilde{\ba} = \argmax_{\ba \in \ball^{m+L}} ~ \br(\ba)
\end{equation}
is well-defined.
The next proposition quantifies the approximation error incurred when playing cyclically $\tilde{\ba}$ instead of the optimal sequence of actions $(a^*_t)_{t \le T}$ defined in \eqref{eq:opt_sequence}.
A critical quantity to establish this result is the maximal (and minimal) instantaneous reward one can obtain.
To this end, we introduce the notation $R = \sup_{a_1,\ldots,a_{m+1} \in \mathcal{A}} \big|\langle a_{m+1}, A(a_1,\ldots,a_m) \theta^*\rangle\big|$.
Note that in \eqref{eq:bound_R} we provide a bound on $R$ in terms of $m$ and $\gamma$.
We now state our approximation result, and show that it is tight up to constant.

\begin{restatable}{proposition}{propapproxcyclic}\label{prop:approx_cyclic}
For any $m,L \ge 1$, let $\tilde{\ba}$ be the block of $m+L$ actions defined in \eqref{eq:best_block} and $(\tilde{r}_t)_{t=1}^T$ be the expected rewards collected when playing cyclically $\tilde{\ba}$. We have
\begin{equation}\label{eq:approx_cyclic}
\OPT - \sum_{t=1}^T \tilde{r}_t \le \frac{2mR}{m+L}\,T~.
\end{equation}
\end{restatable}

The dependence on the cycle length $L$ of the right-hand side of \eqref{eq:approx_cyclic} is as expected: by increasing $L$, the expected reward of the cyclic policy gets closer to $\OPT$.
In addition, note that for $m=0$ we recover the stationary behaviour.
In this case, there are no long-term effects and the performance is oblivious to the block length,
so that we recover $\sum_t \tilde{r}_t = \OPT$ independently of $L$.
Next, we show that \Cref{prop:approx_cyclic} is tight up to constants.

\begin{proposition}[Tight approximation]\label{prop:cyclic_tight}
For any $m,L \ge 1$ and $\gamma \le 0$, let $\tilde{\ba}$ be the block of $m+L$ actions defined in \eqref{eq:best_block} and $(\tilde{r}_t)_{t=1}^T$ be the expected rewards collected when playing cyclically $\tilde{\ba}$.
Then, there exists a choice of $\mathcal{A}$ and $\theta^*$ such that
\begin{equation}\label{eq:approx_cyclic_other}
\OPT - \sum_{t=1}^T \tilde{r}_t \ge  \frac{mR}{m+L}\,T~.
\end{equation}
\end{proposition}
\begin{proof}
Let $d=m+1$, $\mathcal{A} = \{0_d\} \cup (e_k)_{k \le d}$, $\theta^* = (1/\sqrt{d}, \ldots, 1/\sqrt{d})$, and $\gamma \le 0$.
For simplicity, we note the basis modulo $d$, i.e., $e_{k+d} = e_k$ for any $k \in \mathbb{N}$.
Note that for any $a_1, \ldots, a_{m+1} \in \A$ we have $\big|\langle a_{m+1}, A_m \theta^*\rangle\big| \le \|a_{m+1}\|_1 ~ \|A_m\theta^*\|_\infty \le 1/\sqrt{d}$, such that one can take $R=1/\sqrt{d}$.
Observe now that the strategy which plays cyclically $e_1, \ldots, e_d$ collects a reward of $1/\sqrt{d}$ at each time step, which is optimal, such that $\OPT = T/\sqrt{d}$.
Further, it is easy to check that block $\tilde{\ba}$, composed of $m$ pulls of $0_d$ followed by $e_1, \ldots, e_L$ satisfies $\tilde{r}(\tilde{\ba}) = L/\sqrt{d}$, which is optimal for similar reasons.
Playing cyclically $\tilde{\ba}$, one gets a reward of $L/\sqrt{d}$ every $m+L$ pulls.
In other terms, we have
\[
\OPT - \sum_{t=1}^T \tilde{r}_t = \frac{T}{\sqrt{d}} - \frac{L}{m+L}\frac{T}{\sqrt{d}} = \frac{m}{m+L}\frac{T}{\sqrt{d}} = \frac{mR}{m+L}\,T~.
\]
\end{proof}

Upper bounds on $R$ are easy to obtain.
Let $a_1, \ldots, a_{m+1} \in \A$, and $A_m = A(a_1,\ldots,a_m)$, we have
\begin{equation}\label{eq:bound_R}
|r_m| = \big|\langle a_{m+1}, A_m \theta^*\rangle\big| \le \|a_{m+1}\|_2 ~ \|A_m\theta^*\|_2 \le \|A_m\|_* ~ \|\theta^*\|_2 \le (m+1)^{\gamma^+}\,,
\end{equation}
such that one can take $R = (m+1)^{\gamma^+}$.
Note that any other choice of dual norms could have been used to upper bound $\big|\langle a_{m+1}, A_m \theta^*\rangle\big|$, as done in \Cref{prop:cyclic_tight}.
For simplicity, we restrict ourselves to the Euclidean norm from now on, and use $R = (m+1)^{\gamma^+}$.

\begin{remark}[On the necessity of optimizing over the first actions.]
We highlight that optimizing over the first $m$ actions in~\Cref{eq:best_block} is necessary, as there exists no such ``pre-sequence'' which is universally optimal.
Indeed, let $A_t$ and $A'_t$ be the memory matrices generated by $a_1 \ldots a_{m+L}$ and $a'_1 \ldots a'_m \, a_{m+1} \ldots a_{m+L}$ respectively.
It is immediate to check that if the \mbox{pre-sequence} $a_1 \ldots a_m$ is better than $a'_1 \ldots a'_m$ with respect to some model $\theta \in \mathbb{R}^d$, i.e., if we have $\sum_{t=m+1}^{m+L} \langle a_t, A_{t-1}\theta\rangle \ge \sum_{t=m+1}^{m+L} \langle a_t, A'_{t-1}\theta\rangle$, then the opposite holds true for $-\theta$.
Hence, one cannot determine \textit{a priori} a good pre-sequence and has to optimize for it.
\end{remark}


\subsection{Estimation}
\label{sec:estimation}
The next step now consists in building a sequence of blocks with small regret against $\tilde{\ba}$.
As detailed below, this reduces to a stationary linear bandit problem, with a specific action set.
After showing an initial naive solution, we provide a refined approach which exploits the structure of the latent parameter and enjoys improved regret guarantees.

\paragraph{A naive approach.}
We introduce some notation first.
Let $\btheta^* = (0_d, \ldots, 0_d, \theta^*, \ldots, \theta^*) \in \mathbb{R}^{d(m+L)}$ be the vector concatenating $m$ times $0_d$ and $L$ times $\theta^*$.
Inspired by the right-hand side in \eqref{eq:rwd_lower_bound}, we introduce the subset of $\mathbb{R}^{d(m+L)}$ composed of the blocks $\bb = b_1 \ldots b_{m+L}$ whose actions are of the form $b_i = A_{i-1} a_i$ for some block $\ba \in \A^{m+L}$.
Formally, let
\[
\bball = \left\{\bb \in \mathbb{R}^{d(m+L)}\colon \exists\,\ba \in \A^{m+L} \text{ such that }
\begin{cases}
b_i = a_i &1 \le i \le m\\
b_i = A_{i-1}a_i &m+1 \le i \le m+L
\end{cases}~~
\right\}\,,
\]
where the $(A_i)_{i=m+1}^{m+L-1}$ are the memory matrices generated from $\ba$.
Equipped with this notation, it is easy to see that for any $\ba \in \A^{m+L}$ and the corresponding $\bb \in \bball$ we have $\br(\ba) = \langle \bb, \btheta^*\rangle$.
Therefore, estimating $\tilde{\bb}$ (the block in $\bball$ associated to $\tilde{\ba}$) reduces to a standard stationary linear bandit problem in $\mathbb{R}^{d(m+L)}$, with parameter $\btheta^*$ and feasible set $\bball$.
In other words, we have transformed the nonstationarity of the rewards into a constraint on the action set.
Running OFUL \citep{abbasi2011improved} then amounts to playing at time step $t = \tau(m+L)$, the block $\ba_\tau \in \A^{m+L}$, whose associated block $\bb_\tau$ in $\bball$ satisfies
\begin{equation}\label{eq:oful_block}
\bb_\tau = \argmax_{\bb \in \bball} \sup_{\btheta \in \bC_{\tau-1}} ~ \langle \bb, \btheta \rangle\,,
\end{equation}
where $\bC_\tau = \big\{\btheta \in \mathbb{R}^{d(m+L)} \colon \big\|\hat{\btheta}_\tau - \btheta\big\|_{\bV_\tau} \le \bbeta_\tau(\delta)\big\}$, with $\bbeta_\tau(\delta)$ defined in \Cref{eq:big_beta}, $\bV_\tau = \sum_{\tau'=1}^\tau \bb_{\tau'}\bb_{\tau'}^\top + \lambda I_{d(m + L)}$\,, $\by_\tau = \sum_{i=m+1}^{m+L} y_{\tau, i}$\,, using $y_{\tau, i}$ to denote the reward obtained by the $i^\text{th}$ action of block $\tau$, and
\begin{equation}\label{eq:est_theta_block}
\hat{\btheta}_\tau = \bV_\tau^{-1}\left(\sum_{\tau'=1}^\tau \by_{\tau'}\bb_{\tau'}\right)\,.
\end{equation}
Noticing that $\|\btheta^*\|_2^2 \le L$, that for any block $\bb \in \bball$ we have $\|\bb\|_2^2 \le m + L (m +1)^{2\gamma^+}$ and $\langle \btheta^*, \bb\rangle \le L(m+1)^{\gamma^+}$, and adapting the OFUL's analysis, we get the following regret bound.

\begin{proposition}\label{prop:adaptation_exp}
Let $\lambda \in [1, d]$, $L \ge m$, and $\ba_\tau$ be the blocks of actions in $\mathbb{R}^{d(m+L)}$ associated to the $\bb_\tau$ defined in \eqref{eq:oful_block}.
Then we have
\[
\mathbb{E}\left[\sum_{\tau = 1}^{T/(m+L)} \br(\tilde{\ba}) - \br(\ba_\tau)\right] = \tilde{\mathcal{O}}\Big(dL^{3/2}(m+1)^{\gamma^+}\sqrt{T}\Big)~.
\]
\end{proposition}

In the stationary case, i.e., when $m=0$ and $L=1$, the block approach coincide with OFUL and we do recover (up to log factors) the $\mathcal{O}(d\sqrt{T})$ bound for standard linear bandits.
Note that in \Cref{prop:adaptation_hp} in the Supplementary Material we prove a more general high-probability bound, which also specializes to known results for linear bandits in the stationary case.\looseness-1

\paragraph{A refined approach.}
As revealed by \eqref{eq:est_theta_block}, the previous approach is wasteful.
Indeed, while the relevant model to estimate is $\theta^* \in \mathbb{R}^d$, the $\hat{\btheta}_\tau$ are estimators of the concatenated vector $\btheta^* \in \mathbb{R}^{d(m+L)}$, with degraded accuracy due to the increased dimension.
Similarly, this method only uses the sum of rewards obtained by a block, while finer-grained information is available, namely the rewards obtained by each individual action in the block.
Driven by these considerations, let $\ba_\tau = a_{\tau, 1} \ldots a_{\tau, m+L}$ be the block of actions played at block time step $\tau$, $A_{\tau, i-1} = A(a_{\tau, i-m}, \ldots, a_{\tau, i-1})$, and $b_{\tau, i} = A_{\tau, i-1} a_{\tau, i}$ for $i \ge m$.
We propose to compute instead
\begin{equation}\label{eq:est_theta}
\hat{\theta}_\tau = V_\tau^{-1}\left(\sum_{\tau'=1}^\tau \sum_{i=m+1}^{m+L} y_{\tau', i}\,b_{\tau', i}\right)\,,
\end{equation}
where
$V_\tau = \sum_{\tau'=1}^\tau \sum_{i=m+1}^{m+L} b_{\tau', i}b_{\tau', i}^\top + \lambda I_d$.
In words, $\hat{\theta}_\tau$ is the standard regularized least square estimator of $\theta^*$ when only the last $L$ rewards of each block of size $m+L$ are considered.
Note however that the $\hat{\theta}_\tau$ are only computed every $m+L$ rounds.
Indeed, recall that regret is computed here at the block level, such that at each block time step $\tau$ the learner chooses upfront an entire block to play, preventing from updating the estimates between the individual actions of the block.
Following the principle of optimism in the face of uncertainty, a natural strategy then consists in playing
\begin{equation}\label{eq:oful_indiv}
\ba_\tau = \argmax_{a_{\tau, i} \in \mathcal{A}} \sup_{\theta \in \mathcal{C}_{\tau - 1}} ~ \sum_{i=1}^L \langle a_{\tau, i}, A_{\tau, i-1} \theta\rangle\,,
\end{equation}
where $\mathcal{C}_\tau = \big\{\theta \in \mathbb{R}^d \colon \big\|\hat{\theta}_\tau - \theta\big\|_{V_\tau} \le \beta_\tau(\delta)\big\}$, for some $\beta_\tau(\delta)$ defined in \eqref{eq:indiv_beta}.
Expressed in terms of $\bb_\tau$, the estimate \eqref{eq:oful_indiv} corresponds to
\begin{equation}\label{eq:optim_oo}
\bb_\tau = \argmax_{\bb \in \bball} \sup_{\btheta \in \bD_{\tau-1}} ~ \langle \bb, \btheta \rangle\,,
\end{equation}
where $\bD_{\tau} = \big\{\btheta \in \mathbb{R}^{d(m+L)} \colon \exists \theta \in \mathcal{C}_\tau \text{  such that } \btheta = (0_d, \ldots, 0_d, \theta, \ldots, \theta)\big\}$.
In words, this estimate is similar to \eqref{eq:oful_block}, except that we use the improved confidence set $\bD_\tau$ that leverages the structure of $\btheta^*$.
A dedicated analysis to deal with the fact that the estimates $\hat{\theta}_\tau$ are not ``up to date'' for actions inside the block then allows to bound the regret of the sequence $\ba_\tau$ against the optimal $\tilde{\ba}$.
Setting the block size $L$ in order to balance this bound with the approximation error of \Cref{prop:approx_cyclic} yields the final regret bound.
\smallskip

\begin{theorem}\label{thm:regret_exp}
Let $\lambda \in [1, d]$, and $\ba_\tau$ be the blocks of actions in $\mathbb{R}^{d(m+L)}$ defined in \eqref{eq:oful_indiv}.
Then we have
\[
\mathbb{E}\left[\sum\nolimits_{\tau = 1}^{T/(m+L)} \br(\tilde{\ba}) - \br(\ba_\tau)\right] = \tilde{\mathcal{O}}\Big(d L (m+1)^{\gamma^+} \sqrt{T}\Big)\,.
\]

Suppose that $m \ge 1$, $T \ge d^2m^2 + 1$, and set $L = \big\lceil\sqrt{m/d}~T^{1/4}\big\rceil - m$.
Let $r_t$ be the rewards collected when playing $\ba_\tau$ as defined in \eqref{eq:oful_indiv}.
Then we have
\[
\OPT - \mathbb{E}\left[\sum\nolimits_{t=1}^T r_t \right] = \tilde{\mathcal{O}}\left(\sqrt{d}~(m+1)^{\frac{1}{2} + \gamma^+}\,T^{3/4}\right)\,.
\]
When $m=0$ (i.e., in the stationary case), setting $L=1$ recovers the OFUL bound.
\end{theorem}
%
%
Note that the dependence in $L$ has been reduced from $L^{3/2}$ to $L$ thanks to the improved confidence sets.
Solving the approximation-estimation tradeoff using instead \Cref{prop:adaptation_exp} would provide an overall regret bound of order $d^{2/5} (m+1)^{\frac{3}{5} + \gamma^+}~T^{4/5}$, worse than the bound provided by the second claim of \Cref{thm:regret_exp}.

\begin{remark}[An over-optimistic variant]\label{rmk:over_opt}
Note that $\bD_{\tau} = \big\{\btheta \in \mathbb{R}^{d(m+L)} \colon \exists \theta \in \mathcal{C}_\tau \text{  such that } \btheta = (0_d, \ldots, 0_d, \theta, \ldots, \theta)\big\}$ is not the only improved confidence set that one can build from $\mathcal{C}_\tau$.
Indeed, it is immediate to check that our proof remains unchanged if one uses instead $\bD^\text{opt}_{\tau} = \big\{\btheta \in \mathbb{R}^{d(m+L)} \colon \exists \, \theta_1, \ldots, \theta_L \in \mathcal{C}_\tau \text{  such that } \btheta = (0_d, \ldots, 0_d, \theta_1, \ldots, \theta_L)\big\}$.
Optimizing \eqref{eq:optim_oo} over $\bD^\text{opt}_{\tau-1}$ and not $\bD_{\tau-1}$ creates an over-optimistic block version of the UCB, composed of the sum of the UCBs of the single-actions in the block, although the latter might be attained at different models $\theta_i$, while we know that $\btheta^*$ is the same model $\theta^*$ repeated $L$ times.
Still, since each $\theta_i$ is estimated in the confidence set $\mathcal{C}_{\tau-1}$ of reduced dimension, the guarantees are unchanged.
In the rest of the paper, we refer to this variant as the \emph{over-optimistic} version of OFUL-memory, denoted by \textnormal{\texttt{O3M}}.
\end{remark}

%
Finding a lower bound matching \Cref{thm:regret_exp} for abitrary values of $m$ and $\gamma$ remains an open problem.
%
%
Yet, \Cref{prop:cyclic_tight} shows that the control of the approximation error provided by \Cref{prop:approx_cyclic} is optimal up to constants.
Moreover, our estimation error is tight in general, as in the stationary case (i.e., $m=0$) the upper bound in \Cref{thm:regret_exp} matches the lower bound for stationary linear bandits, see e.g., \cite[Theorems~24.1 and 24.2]{lattimore2020bandit}. 
%


As we can see from the optimal choice of $L$ in \Cref{thm:regret_exp}, OFUL-memory requires the knowledge of the horizon $T$, the memory size $m$, and the exponent $\gamma$, which might all be unknown in practice.
If adaptation to $T$ can be achieved by using the doubling trick, adaptation to $m$ and $\gamma$ is more involved.
The purpose of the next subsection is to show that OFUL-memory can be wrapped by a model selection algorithm for learning $m$ and $\gamma$ while providing good regret guarantees.


\subsection{Model Selection}
\label{sec:model_selection}

%
%
%
%
In the absence of prior knowledge on the nature of the nonstationary mechanism at work, a natural idea consists in instantiating several LBMs with different values of $\gamma$ and running a model selection algorithm for bandits \citep{foster2019model,cutkosky2020upper,pacchiano2020model}.
In bandit model selection, where a master algorithm runs the different LBMs, the adaptation to the memory size $m$ becomes more complex.
Indeed, the different putative values for $m$ induce different block sizes (see \Cref{thm:regret_exp}) which perturb the time and reward scales of the master algorithm.
For instance, bandits with larger block length will collect more rewards per block, although they might not be more efficient on average.
Our solution consists in feeding the master algorithm with averaged rewards. 
One may then control the true regret (i.e., not averaged) of the output sequence, against a scaled version of the optimal sequence through \Cref{lem:average} in \Cref{apx:corollary}.
Combining this result with \Cref{thm:regret_exp} and \citep[Corollary~2]{cutkosky2020upper} yields the following corollary, that bounds the regret of OFUL-memory with model selection.

\begin{restatable}{corollary}{cormodelselection}\label{cor:model_selection}
Consider an instance of LBM with unknown parameters $(m_\star, \gamma_\star)$.
Assume a bandit combiner is run on $N \le d \sqrt{m_\star}$ instances\footnote{The condition on $N$ is merely used to simplify the bound and can be dropped altogether, see \Cref{apx:corollary}.} of OFUL-memory (\Cref{alg:bc}), each using a different pair of parameters $(m_i, \gamma_i)$ from a set $\mathcal{S} = \big\{(m_1, \gamma_1),\ldots,(m_N, \gamma_N)\big\}$ such that $(m_\star, \gamma_\star) \in \mathcal{S}$.
Let $M = (\max_j m_j) / (\min_j m_j)$.
Then, for all $T \ge (m_\star + 1)^{2\gamma_{\star}^+} / m_\star d^4$, the expected rewards $\big(r_t^\textnormal{bc}\big)_{t=1}^T$ of the bandit combiner satisfy
\[
\frac{\OPT}{\sqrt{M}} - \mathbb{E}\left[\sum_{t=1}^T r^\textnormal{bc}_t\right] ~=~ \tilde{\mathcal{O}}\Big(  M\,d\,(m_\star + 1)^{1+\frac{3}{2}\gamma_{\star}^+} \,T^{3/4}\Big)\,.
\]
\end{restatable}

\section{Algorithms}
\label{sec:algo}

\begin{algorithm}[!t]
\SetKwInOut{Input}{input}
\SetKwInOut{Init}{init}
\SetKwInOut{Parameter}{Param}
\caption{\texttt{OFUL-memory (OM, O3M)}}
\Input{~~action space $\mathcal{A} \subset \mathbb{R}^d$, memory size $m$, exponent $\gamma$, regularization parameter $\lambda$, horizon $T$.}\vspace{0.15cm}
\Init{~~set $L = \sqrt{m/4\,d}~T^{1/4} - m$, \, $\hat{\theta}_0 = 0_d$, ~\,$V_0 = \lambda I_d$, \, $\beta_0 = 0$.}\vspace{0.15cm}
\For{$\tau = 1, \dots, T/(m+L)$}{\vspace{0.2cm}{
%
%
\tcp{\hspace{0.1cm}OM\hspace{6.1cm}// \,O3M}\vspace{-0.25cm}
$\displaystyle \ba_\tau = \argmax_{a_{\tau, i} \in \mathcal{A}} \sup_{\theta \in \mathcal{C}_{\tau - 1}} ~ \sum_{i=1}^L \langle a_{\tau, i}, A_{\tau, i-1} \theta\rangle$ \quad or \quad $\displaystyle \ba_\tau = \argmax_{a_{\tau, i} \in \mathcal{A}} \sup_{\theta_i \in \mathcal{C}_{\tau - 1}} ~ \sum_{i=1}^L \langle a_{\tau, i}, A_{\tau, i-1} \theta_i\rangle$\vspace{0.2cm}

\tcp{Play and update confidence set}
Play $\ba_\tau$, collect $y_{\tau,1}, \dots, y_{\tau, m+L}$, and compute $\mathcal{C}_\tau$, i.e., $\hat{\theta}_\tau$, $V_\tau$, and $\beta_\tau$ via \eqref{eq:est_theta} and \eqref{eq:indiv_beta}.\vspace{0.1cm}
}
}
\label{alg:OFUL-memory-algo}
\vspace{-0.1cm}
\end{algorithm}

In this section, we discuss the practical implementations of our approaches, OFUL-memory (\texttt{OM}) and over-optimistic OFUL-memory (\texttt{O3M}, see \Cref{rmk:over_opt}), both summarized in \Cref{alg:OFUL-memory-algo}.

\paragraph{Maximizing the UCBs.}
We start by making explicit the UCBs used in \texttt{OM} and \texttt{O3M}, see \eqref{eq:optim_oo}, optimized over $\bD_{\tau}$ or $\bD^\text{opt}_{\tau}$.
Using the formula for $\mathcal{C}_\tau$ one can check that they are given by $\mathrm{UCB}_\tau(\ba) = \sum_{j=m+1}^{m+L} \big\langle a_j, A_{j-1} \hat{\theta}_{\tau-1} \big\rangle + B(\ba)$, where $B(\ba) = \beta_{\tau - 1} \big\| \sum_{j=m+1}^{m+L} A_{j-1}^\top a_j \big\|_{V^{-1}_{\tau - 1}}$ for \texttt{OM} and $B(\ba) = \beta_{\tau - 1} \big\| A_{j-1}^\top a_j \big\|_{V^{-1}_{\tau - 1}}$ for \texttt{O3M}.
The two UCBs only differ in their exploration bonuses.
Note that by the triangle inequality, we have $\mathrm{UCB}^\text{OM}_\tau(\ba) \le \mathrm{UCB}^\text{O3M}_\tau(\ba)$ for any $\ba$.\looseness-1

Thanks to this closed form in terms of $\ba$, it is possible to solve $\argmax_{\ba} \mathrm{UCB}_{\tau}(\ba)$, using gradient ascent.
Note, however, that proving theoretical guarantees on the quality of the solution obtained can be difficult in general, as shown by the following simple example.
Let $d=1$, $m=1$, $L=1$, and $\gamma=-1$, such that $A(x) = (I_d+xx^\top)^{-1} = 1/(1+x^2)$.
Then we have $\mathrm{UCB}_\tau(x, y) = y\,\hat{\theta}_\tau/(1+x^2)$, which is neither convex nor concave.
An interesting research direction would consist in bounding the optimization error of gradient ascent, so that it could be included in the tradeoff with the approximation and estimation errors in order to set $L$ in the best possible way.


\paragraph{Bandit combiner.}
Our bandit combiner algorithm builds on the approach in \cite{cutkosky2020upper}. However, we made some significant modifications to both the algorithm and its analysis in order to take into account the switching between blocks of different size and the nonlinear scaling of the rewards with the block size in the rising case.
Due to space constraints, the pseudo-code of the algorithm is deferred to \Cref{apx:bc}.

\section{Experiments}
\label{sec:experiments}
We perform experiments to validate the theoretical performance of \texttt{OM} and \texttt{O3M} (\Cref{alg:OFUL-memory-algo}). Similarly to \citep{warlop2018fighting}, we work with synthetic data because of the counterfactual nature of the learning problem in bandits. 
Unless stated otherwise, we set $d=3$ while $\theta^* \in \mathbb{R}^{d}$ is generated uniformly at random with unit norm. The rewards are generated according to \eqref{eq:reward} and \eqref{eq:A-def}, and perturbed by Gaussian noise with standard deviation $\sigma=1/10$.
Note that \Cref{sec:more-exp} contains additional experiments.\looseness-1

\par{\bf Rotting with Bandit Combiner.} We start by analyzing the rotting scenario with $m = 2$ and $\gamma = -3$. We measure the performance in terms of the cumulative reward averaged over $5$ runs (this is enough because the variance is small). 
In \Cref{fig:exp} (left pane) we compare the performance of \texttt{O3M} against oracle greedy, vanilla OFUL, and two instances of Bandit Combiner (\Cref{alg:bc}, see \Cref{apx:bc} in the Supplementary Material). The first instance, Combiner $\gamma$, works in the setting where the misspecified parameter is $\gamma$ and the algorithm is run over 
the set $\{-4, -3, -2, -1, 0\}$ of possible values for $\gamma$ with the true value being $-3$.
The second instance, Combiner $m$, tests the setting where the misspecified parameter is $m$. In this case the algorithm is run over the set $\{0, 2, 3\}$ of possible values for $m$ with the true value being $2$.
The results---see \Cref{fig:exp} (left pane)---show that \texttt{O3M} is able to plan the actions in the block ensuring that a good arm is not played right away if a higher reward can be obtained later on in the block. This means that \texttt{O3M} is waiting to play certain actions until the corresponding entries of $A$ have been offloaded, preventing $A$ to negatively impact the reward of these actions.
Although learning $m$ proves to be more difficult, which is consistent with the impact of $M = (\max_j m_j)\big/(\min_j m_j)$ in \Cref{cor:model_selection}, Combiner $m$ run on instances of \texttt{O3M} is competitive with \texttt{O3M} run with the true parameters.
Note that with isotropic initialization there is no point in running Combiner $\gamma$ with values of $\gamma$ larger than zero. Indeed, in the isotropic case oracle greedy is optimal, stationary, and with the same optimal action for any $\gamma \ge 0$. The empirical performance of our algorithms in a non-isotropic rising setting is investigated in the next example.


%
\begin{figure*}[!t]
    \centering
    \includegraphics[width=.45\textwidth]{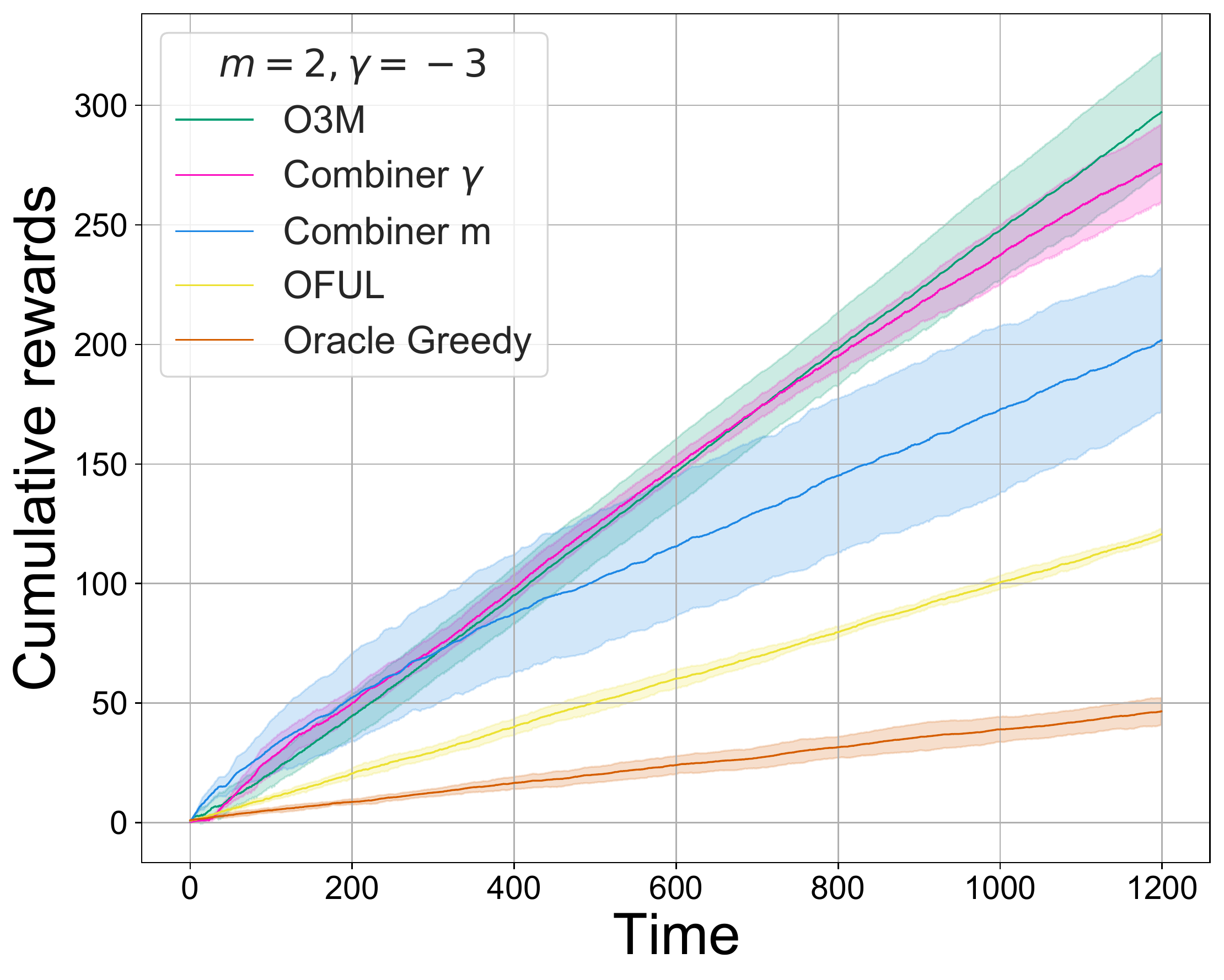}
    \includegraphics[width=.45\textwidth]{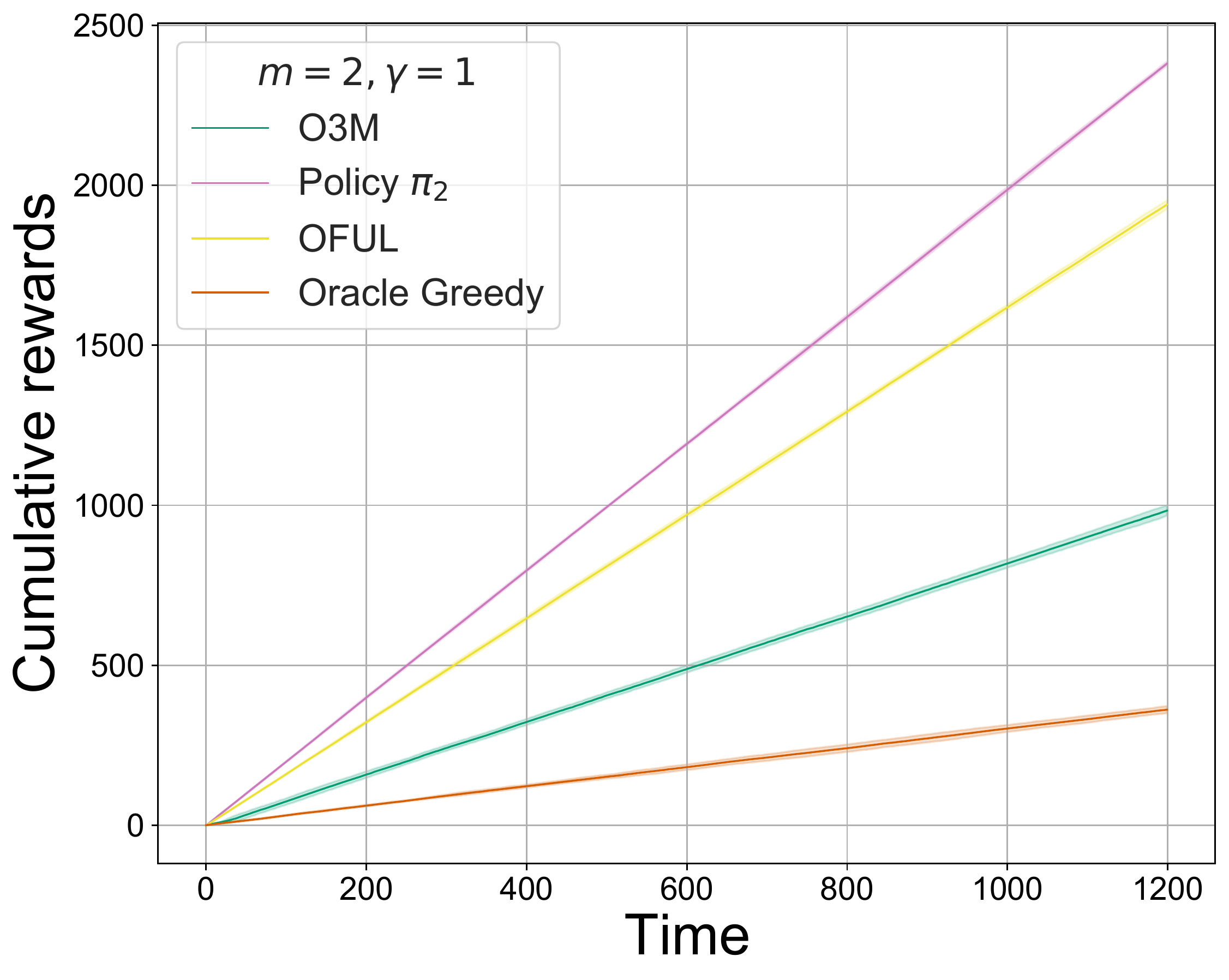}
    \caption{Cumulative rewards in rotting (left) and rising with non-isotropic initialization (right) cases.}
    \label{fig:exp}
\end{figure*}
%
%

\par{\bf Rising with non-isotropic initialization.} When $\gamma > 0$ (rising setting) and
$A_0 \neq I_d$ (non-isotropic initialization), there are instances for which oracle greedy is suboptimal, as we show next. Let $d=2$, $m=2$, $\gamma=1$, $A_0 = \begin{pmatrix}1 & 0\\0&0\end{pmatrix}$,
%
%
%
%
and $\theta^* = ( \sqrt{\epsilon}, \sqrt{1 - \epsilon})$.
With these choices,
oracle greedy starts to pull action $e_1 = (1, 0)$ and will always play it, obtaining a cumulative reward of $T (1 + m) \sqrt{\epsilon}$. Instead, a better strategy would be to play $e_2 = (0, 1)$ all the time, collecting a cumulative reward of $T m \sqrt{1 - \epsilon}$. We call this strategy $\pi_2$ and in \Cref{fig:exp} (right pane) we compare the performance of \texttt{O3M} with oracle greedy, $\pi_2$, and OFUL. Here OFUL performs well because the optimal action is stationary and, unlike oracle greedy, OFUL can use exploration to discover that $e_2$ is better than $e_1$.

\section{Conclusions and open problems}
We introduced and analyzed a nonstationary generalization of linear bandits using a fixed-size memory.
Future research directions include: generalizing our approach to other kinds of memory matrices, including the UCB optimization error into the tradeoff to tune $L$, deriving matching lower bounds.\looseness-1




\bibliography{ref}
\bibliographystyle{apalike}

\appendix
\onecolumn
\section{Technical Proofs}
\label{apx:proof}

We gather in this section the proofs omitted in the core text.

\subsection{Proof of Proposition \ref{prop:greedy_bad}}

\propgreedybad*

\begin{proof}
We build two instances of LBM, one rotting, one rising, in which the oracle greedy strategy suffers linear regret.
We highlight that the other strategy exhibited, which performs better than oracle greedy, may not be optimal.

{\bf Rotting instance.}
Let $\mathcal{A} = \ball$, $\theta^* = e_1$, $m=d-1$, and $A$ such that
\[
A(a_1, \ldots, a_m) = \left(I_d + \sum_{s=1}^m a_s a_s^\top\right)^{-\gamma}\,,
\]
for some $\gamma > 0$ to be specified later.
Oracle greedy, which plays at each time step $a^\mathrm{greedy}_t = \argmax_{a \in \mathcal{A}} \langle a, A_{t-1}\theta^*\rangle$, constantly plays $e_1$.
After the first $m$ pulls, it collects a reward of $1/d^\gamma$ at every time step.
On the other side, the strategy that plays cyclically the block $e_1 \ldots e_d$ collects a reward of $1$ every $d = m+1$ time steps, i.e., an average reward of $1/d$ per step.
Hence, up to the transitive first $m$ puuls, the cumulative reward of oracle greedy after $T$ rounds is $T/d^{\gamma}$, and that of the cyclic policy is $T/d$. The regret of oracle greedy is thus at least
\[
T\left(\frac{1}{d} - \frac{1}{d^{\gamma}}\right)\,,
\]
which is linear for $\gamma > 1$.

{\bf Rising instance.}
Let $m \ge 1$, $d=2$, $\mathcal{A} = \mathcal{B}_2$, $\theta^* = (\varepsilon,1)$ where $\varepsilon > 0$ is to be specified later, and $A$ such that
\[
A(a_1, \ldots, a_m) = \begin{pmatrix}1 & 0\\0&0\end{pmatrix} + \sum_{s=1}^m a_s a_s^\top\,.
\]
Oracle greedy constantly plays $e_1$ collecting a reward of $(m+1)\theta^*_1$ from round $m+1$ onward.
On the other side, the strategy that plays constantly $e_2$ collects a reward of $m\theta^*_2$ from round $m+1$ onward.
Hence, the regret of oracle greedy from round $m+1$ onward is at least $(T - m) [m - (m + 1)\varepsilon]$, which is linear for $\varepsilon < m/(m+1)$.
\end{proof}

\subsection{Proof of Proposition \ref{prop:approx_cyclic}}

\propapproxcyclic*

\begin{proof}
Recall that the optimal sequence is denoted $(a^*_t)_{t=1}^T$ and collects rewards $(r^*_t)_{t=1}^T$.
Let $L > 0$; by definition, there exists a block of actions of length $L$ in $(a^*_t)_{t=1}^T$ with average expected reward higher that $\opt/T$.
Let $t^*$ be the first index of this block, we thus have
$(1/L)\sum_{t = t^*}^{t^* + L - 1} r^*_t \ge \opt/T$.
However, this average expected reward is realized only using the initial matrix $A_{t^* - 1}$, generated from $a^*_{t^* - 1}, \ldots, a^*_{t^* - m}$.
Let $\ba^* = a^*_{t^* - m}, \ldots, a^*_{t^* + L - 1}$ of length $m + L$.
Note that, by definition, we have that $\br(\tilde{\ba}) \ge \br(\ba^*) = \sum_{t = t^*}^{t^* + L - 1} r^*_t \ge L ~ \opt/T$.
Furthermore, by \eqref{eq:bound_R}, when playing cyclically $\tilde{\ba}$ one obtains at least a reward of $-R$ in each one of the first $m$ pulls of the block.
Collecting all the pieces, we obtain
\begin{align}
\sum_{t=1}^T \tilde{r}_t &\ge \frac{T}{m+L}\Big(-mR + \br(\tilde{\ba})\Big)\nonumber\\
&\ge \frac{T}{m+L}\Big(-mR + \br(\ba^*)\Big)\nonumber\\
&\ge \frac{T}{m+L}\left(-mR + L\,\frac{\opt}{T}\right)\nonumber\\
&=\frac{L}{m+L}\opt - \frac{mR}{m+L}\,T\,\nonumber\\
&\ge\frac{L}{m+L}\opt + \frac{m}{m+L}\opt - \frac{mR}{m+L}\,T - \frac{mR}{m+L}\,T\,\label{eq:apply_bound_rwd}\\
&= \OPT - \frac{2mR}{m+L}\,T\,,\nonumber
\end{align}
where \eqref{eq:apply_bound_rwd} derives from $\OPT \le R T$.
\end{proof}

\subsection{Proof of Proposition \ref{prop:adaptation_exp}}
\label{apx:proof_adaptation}

We prove the (stronger) high probability version of \Cref{prop:adaptation_exp}.

\begin{proposition}\label{prop:adaptation_hp}
Let $\lambda \ge 1$, $\delta \in (0,1)$, and $\ba_\tau$ be the blocks of actions in $\mathbb{R}^{d(m+L)}$ associated to the $\bb_\tau$ defined in \eqref{eq:oful_block}.
Then, with probability at least $1-\delta$ we have
\begin{align*}
\sum_{\tau = 1}^{T/(m+L)} \br(\tilde{\ba}) - \br(\ba_\tau) &\le 4L(m+1)^{\gamma^+} \, \sqrt{Td ~ \ln \left(1 + \frac{T(m+1)^{2\gamma^+}}{d(m+L)\lambda}\right)}\\
&\hspace{1cm}\cdot\left(\sqrt{\lambda L} + \sqrt{\ln\left(\frac{1}{\delta}\right) + d(m + L) \, \ln\left(1 + \frac{T (m+1)^{2\gamma^+}}{d (m + L)\lambda}\right)}\right)\,.
\end{align*}
\end{proposition}

\begin{proof}
The proof essentially follows that of \citep[Theorem~3]{abbasi2011improved}.
The main difference is that our version of OFUL operates at the block level. This implies a smaller time horizon, but also and increased dimension and an instantaneous regret $\langle\tilde{\bb}, \btheta^*\rangle - \langle\bb_\tau, \btheta^*\rangle$ upper bounded by $2L (m + 1)^{\gamma^+}$ instead of $1$.
%
%
We detail the main steps of the proof for completeness.
Recall that running OFUL in our case amounts to compute at every block time step $\tau$
\[
\hat{\btheta}_\tau = \bV_\tau^{-1}\Bigg(\sum_{\tau'=1}^\tau \by_{\tau'}\,\bb_{\tau'}\Bigg)\,,
\]
where
\[
\bV_\tau = \sum_{\tau'=1}^\tau \bb_{\tau'}\bb_{\tau'}^\top + \lambda I_{d(m + L)}\,, \qquad \text{and} \qquad \by_\tau = \sum_{i=m+1}^{m+L} y_{\tau, i}\,,
\]
since we associate with a block of actions the sum of rewards obtained after time step $m$.
Note that by the determinant-trace inequality, see e.g., \citep[Lemma~10]{abbasi2011improved}, with actions $\bb_\tau$ that satisfy $\|\bb_\tau\|_2^2 \le m + L (m+1)^{2\gamma^+}$ we have
\begin{equation}\label{eq:determinant}
\frac{|\bV_\tau|}{|\lambda I_{d(m+L)}|} \le \left(1 + \frac{\tau(m + L (m+1)^{2\gamma^+})}{d(m+L)\lambda}\right)^{d(m+L)} \le \left(1 + \frac{\tau  (m+1)^{2\gamma^+}}{d\lambda}\right)^{d(m+L)}\,.
\end{equation}

The action played at block time step $\tau$ is the block $\ba_\tau \in \ball^{m+L}$ associated with
\begin{equation}\label{eq:max_ucb}
\bb_\tau = \argmax_{\bb \in \bball} \, \sup_{\btheta \in \bC_{\tau-1}} \left\langle \bb, \btheta\right\rangle\,,
\end{equation}
where
\[
\bC_\tau = \left\{\btheta \in \mathbb{R}^{d(m+L)} \colon \big\|\hat{\btheta}_\tau - \btheta\big\|_{\bV_\tau} \le \bbeta_\tau(\delta)\right\}\,,
\]
with
\begin{equation}\label{eq:big_beta}
\bbeta_\tau(\delta) = \sqrt{2\ln\left(\frac{1}{\delta}\right) + d(m + L) \, \ln\left(1 + \frac{\tau(m+1)^{2\gamma^+}}{d \lambda}\right)} + \sqrt{\lambda L}\,.
\end{equation}

Applying \citep[Theorem~2]{abbasi2011improved} to $\btheta^* \in \mathbb{R}^{d(m+L)}$ which satisfies $\|\btheta^*\|_2 \le \sqrt{L}$ we have that $\btheta^* \in \bC_\tau$ for every $\tau$ with probability at least $1 - \delta$.
Denoting by $\tilde{\btheta}_\tau$ the model that maximizes \eqref{eq:max_ucb}, we thus have that with probability at least $1 - \delta$, the inequality $\langle \tilde{\bb}, \btheta^*\rangle \le \langle \bb_\tau, \tilde{\btheta}_\tau\rangle$ holds for every $\tau$, and consequently
\begin{align*}
\sum_{\tau=1}^{T/(m+L)} \langle\tilde{\bb}, \btheta^*\rangle &-  \langle\bb_\tau, \btheta^*\rangle\\
&\le \sum_{\tau=1}^{T/(m+L)} \min\left\{2L (m+1)^{\gamma^+}\,,\, \langle\bb_\tau, \tilde{\btheta}_\tau - \btheta^*\rangle\right\}\\
&\le \sum_{\tau=1}^{T/(m+L)} \min\left\{2L (m+1)^{\gamma^+}\,,\, \big\| \tilde{\btheta}_\tau - \btheta^* \big\|_{\bV_{\tau-1}} \, \|\bb_\tau\|_{\bV_{\tau-1}^{-1}}\right\}\\
&\le \sum_{\tau=1}^{T/(m+L)} \min\left\{2L (m+1)^{\gamma^+}\,,\, 2\bbeta_\tau(\delta) \, \|\bb_\tau\|_{\bV_{\tau-1}^{-1}}\right\}\\
&\le 2L(m+1)^{\gamma^+}\,\bbeta_{T/(m+L)}(\delta) \sum_{\tau=1}^{T/(m+L)} \min\left\{1\,,\, \|\bb_\tau\|_{\bV_{\tau-1}^{-1}}\right\}\\
&\le 2L(m+1)^{\gamma^+}\,\bbeta_{T/(m+L)}(\delta) ~ \sqrt{\frac{T}{m+L} \sum_{\tau=1}^{T/(m+L)} \min\left\{1\,,\,\|\bb_\tau\|_{\bV_{\tau-1}^{-1}}^2\right\}}\\
&\le 2\sqrt{2}L (m+1)^{\gamma^+}\,\bbeta_{T/(m+L)}(\delta) ~ \sqrt{\frac{T}{m+L} ~ \ln \frac{|\bV_{T/(m+L)}|}{|\lambda I_{d(m+L)}|}}\\
&\le 4L (m+1)^{\gamma^+} \, \sqrt{Td \, \ln \left(1 + \frac{T (m+1)^{2\gamma^+}}{d(m+L)\lambda}\right)}\\
&\qquad\cdot\left(\sqrt{\lambda L} + \sqrt{\ln\left(\frac{1}{\delta}\right) + d(m + L) \, \ln\left(1 + \frac{T(m+1)^{2\gamma^+}}{d(m+L) \lambda}\right)}~\right)\,,
\end{align*}
where we have used \citep[Lemma~11]{abbasi2011improved}, as well as \eqref{eq:determinant} and \eqref{eq:big_beta}.
Note that in the stationary case, i.e., when $m=0$ and $L=1$, we exactly recover \citep[Theorem~3]{abbasi2011improved}.
\Cref{prop:adaptation_exp} is obtained by setting $\lambda \in [1, d]$, $L \ge m$, and $\delta = 1/T$.
\end{proof}

\subsection{Proof of Theorem \ref{thm:regret_exp}}
\label{apx:proof_regret}

We prove the high probability version of \Cref{thm:regret_exp}, obtained by setting $\lambda \in [1, d]$, and $\delta=1/T$.

\begin{theorem}\label{thm:regret_hp}
Let $\lambda \ge 1$, $\delta \in (0,1)$, and $\ba_\tau$ be the blocks of actions in $\mathbb{R}^{d(m+L)}$ defined in \eqref{eq:oful_indiv}.
Then, with probability at least $1-\delta$ we have
\begin{align*}
\sum_{\tau = 1}^{T/(m+L)} \br(\tilde{\ba}) - \br(\ba_\tau) &\le 4L (m+1)^{\gamma^+} ~ \sqrt{T d~ \ln \left(1+\frac{T (m+1)^{2\gamma^+}}{d\lambda}\right)}\\
&\hspace{2.5cm} \cdot \left(\sqrt{\lambda} + \sqrt{\ln\left(\frac{1}{\delta}\right) + d \, \ln\left(1 + \frac{T (m+1)^{2\gamma^+}}{d (m+L)\lambda}\right)}~\right)\,.
\end{align*}

Let $m \ge 1$, $T \ge m^2d^2 + 1$, and set $L = \big\lceil\sqrt{m/d}~T^{1/4}\big\rceil - m$.
Let $r_t$ be the rewards collected when playing $\ba_\tau$ as defined in~\eqref{eq:oful_indiv}.
Then, with probability at least $1-\delta$ we have
\begin{align*}
\OPT - \sum_{t=1}^T r_t &\le 4\sqrt{d}\,(m+1)^{\frac{1}{2}+\gamma^+} ~ T^{3/4}\Bigg[1 + 2\sqrt{\ln \left(1+\frac{T (m+1)^{2\gamma^+}}{d\lambda}\right)}\\
&\hspace{4cm}\cdot\left(\sqrt{\frac{\lambda}{d}} + \sqrt{\frac{\ln(1/\delta)}{d} + \ln\left(1 + 
\frac{T (m+1)^{2\gamma^+}}{d \lambda}
\right)}~\right)\Bigg]\,.
\end{align*}
\end{theorem}

\begin{proof}
The proof is along the lines of OFUL's analysis.
The main difficulty is that we cannot use the elliptical potential lemma, see e.g., \citep[Lemma~19.4]{lattimore2020bandit} due to the delay accumulated by $V_\tau$, which is computed every $m+L$ round only.
Let
\begin{equation}\label{eq:indiv_beta}
\beta_\tau(\delta) = \sqrt{2\ln\left(\frac{1}{\delta}\right) + d \, \ln\left(1 + \frac{\tau (m+1)^{2\gamma^+}}{d \lambda}\right)} + \sqrt{\lambda}\,.
\end{equation}
By \citep[Theorem~2]{abbasi2011improved}, we have with probability at least $1-\delta$ that $\theta^* \in \mathcal{C}_\tau$ for every $\tau$.
It follows directly that $\btheta^* \in \bD_\tau$ for any $\tau$, such that $\langle \tilde{\bb}, \btheta^*\rangle \le \langle \bb_\tau, \tilde{\btheta}_\tau\rangle$, where $\tilde{\btheta}_\tau = (0_d, \ldots, 0_d, \tilde{\theta}_\tau, \ldots, \tilde{\theta}_\tau)$ with $\tilde{\theta}_\tau \in \mathbb{R}^d$ that maximizes \eqref{eq:oful_indiv} over $\mathcal{C}_{\tau -1}$.
It can be shown that the regret is upper bounded by $\sum_\tau \sum_{i=m+1}^{m+L} \langle b_{\tau, i}, \tilde{\theta}_\tau - \theta^*\rangle$.
Following the standard analysis, one could then use
\[
\big\langle b_{\tau, i}, \tilde{\theta}_\tau - \theta^*\big\rangle \le \|b_{\tau, i}\|_{V_{\tau-1}^{-1}} \, \big\|\tilde{\theta}_t - \theta^*\big\|_{V_{\tau-1}}\,.
\]
While the confidence set gives $\big\|\tilde{\theta}_t - \theta^*\big\|_{V_{\tau-1}} \le 2\beta_{\tau-1}(\delta)$, the quantity $\sum_{i=m+1}^{m+L} \|b_{\tau, i}\|_{V_{\tau-1}^{-1}}$ is much more complex to bound.
Indeed, the elliptical potential lemma allows to bound $\sum_t \|a_t\|_{V_{t-1}^{-1}}^2$ when $V_t = \sum_{s \le t} \, a_s a_s^\top + \lambda I_d$.
However, recall that in our case we have $V_\tau = \sum_{\tau'=1}^\tau \sum_{i=m+1}^{m+L} b_{\tau', i}b_{\tau', i}^\top + \lambda I_d$, which is only computed every $m+L$ rounds.
As a consequence, there exists a ``delay'' between $V_{\tau-1}$ and the action $b_{\tau, i}$ for $i \ge m+2$, preventing from using the lemma.
Therefore, we propose to use instead
\begin{equation}\label{eq:new_cs}
\big\langle b_{\tau, i}, \tilde{\theta}_\tau - \theta^*\big\rangle \le \|b_{\tau, i}\|_{V_{\tau, i-1}^{-1}} \, \big\|\tilde{\theta}_t - \theta^*\big\|_{V_{\tau, i-1}}\,, \quad \text{where} \quad V_{\tau, i} = V_{\tau-1} + \sum_{j=m+1}^i b_{\tau, j} b_{\tau, j}^\top\,.
\end{equation}
By doing so, the elliptical potential lemma applies.
On the other hand, one has to control $\big\|\tilde{\theta}_t - \theta^*\big\|_{V_{\tau, i-1}}$, which is not anymore bounded by $2\beta_{\tau-1}(\delta)$ since the subscript matrix is $V_{\tau, i-1}$ instead of $V_{\tau-1}$.
Still, one can show that for any $i \le m+L$ we have
\begin{align}
\big\|\tilde{\theta}_t - \theta^*&\big\|_{V_{\tau, i-1}}^2\nonumber\\
&= \mathrm{Tr}\left(V_{\tau, i-1} ~ \big(\tilde{\theta}_t - \theta^*\big)\big(\tilde{\theta}_t - \theta^*\big)^\top\right)\nonumber\\
&= \mathrm{Tr}\left(\bigg(V_{\tau-1} + \sum_{j=m+1}^{i-1} b_{\tau, j}b_{\tau, j}^\top \bigg) ~ \big(\tilde{\theta}_t - \theta^*\big)\big(\tilde{\theta}_t - \theta^*\big)^\top\right)\nonumber\\
&= \mathrm{Tr}\left(\bigg(I_d + \sum_{j=m+1}^{i-1} \big(V_{\tau-1}^{-1/2} b_{\tau, j}\big)\big(V_{\tau-1}^{-1/2} b_{\tau, j}\big)^\top \bigg) ~ V_{\tau-1}^{1/2}\big(\tilde{\theta}_t - \theta^*\big)\big(\tilde{\theta}_t - \theta^*\big)^\top V_{\tau-1}^{1/2}\right)\nonumber\\
&\le \bigg\|I_d + \sum_{j=m+1}^{i-1} \big(V_{\tau-1}^{-1/2} b_{\tau, j}\big)\big(V_{\tau-1}^{-1/2} b_{\tau, j}\big)^\top\bigg\|_* ~ \mathrm{Tr}\left(V_{\tau-1}^{1/2}\big(\tilde{\theta}_t - \theta^*\big)\big(\tilde{\theta}_t - \theta^*\big)^\top V_{\tau-1}^{1/2}\right)\nonumber\\
&\le \bigg(1 + \sum_{j=m+1}^{i-1} \big\|V_{\tau-1}^{-1/2} b_{\tau, j}\big\|_2^2\bigg) \, \big\|\tilde{\theta}_t - \theta^*\big\|_{V_{\tau-1}}^2\nonumber\\
&\le \left(1 + (L-1)(m+1)^{2\gamma^+}\right) \, \big\|\tilde{\theta}_t - \theta^*\big\|_{V_{\tau-1}}^2\nonumber\\
&\le L (m+1)^{2\gamma^+} ~ \big\|\tilde{\theta}_t - \theta^*\big\|_{V_{\tau-1}}^2\,.\label{eq:bound_delayed_V}
\end{align}

Recalling also that $\langle\tilde{\bb}, \btheta^*\rangle -  \langle\bb_\tau, \btheta^*\rangle \le 2L (m+1)^{\gamma^+}$, we have with probability at least $1-\delta$
\begin{align}
\sum_{\tau=1}^{T/(m+L)} \langle\tilde{\bb}, \btheta^*\rangle &-  \langle\bb_\tau, \btheta^*\rangle\nonumber\\
&\le \sum_{\tau=1}^{T/(m+L)} \min\left\{2L (m+1)^{\gamma^+}\,,\,\langle\bb_\tau, \tilde{\btheta}_\tau - \btheta^*\rangle\right\}\nonumber\\
&= \sum_{\tau=1}^{T/(m+L)} \min\left\{2L (m+1)^{\gamma^+} \,,\, \sum_{i=m+1}^{m+L} ~ \langle b_{\tau, i}, \tilde{\theta}_\tau - \theta^*\rangle\right\}\nonumber\\
&\le \sum_{\tau=1}^{T/(m+L)} \min\left\{2L (m+1)^{\gamma^+}\,,\,\sum_{i=m+1}^{m+L} ~\|b_{\tau, i}\|_{V_{\tau, i-1}^{-1}} \, \big\|\tilde{\theta}_t - \theta^*\big\|_{V_{\tau, i-1}}\right\}\nonumber\\
&\le \sum_{\tau=1}^{T/(m+L)} \min\left\{2L (m+1)^{\gamma^+}\,,\,2\sqrt{L} (m+1)^{\gamma^+} \beta_{\tau-1}(\delta)~\sum_{i=m+1}^{m+L} ~\|b_{\tau, i}\|_{V_{\tau, i-1}^{-1}}\right\}\nonumber\\
%
%
&\le 2L (m+1)^{\gamma^+}\,\beta_{T/(m+L)}(\delta)\sum_{\tau=1}^{T/(m+L)} \sum_{i=m+1}^{m+L} ~ \min\left\{1\,,\,\|b_{\tau, i}\|_{V_{\tau, i-1}^{-1}}\right\}\nonumber\\
&\le 2L (m+1)^{\gamma^+}\,\beta_{T/(m+L)}(\delta) ~ \sqrt{\frac{T\,L}{m+L}\sum_{\tau=1}^{T/(m+L)} \sum_{i=m+1}^{m+L} ~ \min\left\{1\,,\,\|b_{\tau, i}\|^2_{V_{\tau, i-1}^{-1}}\right\}}\nonumber\\
&\le 2\sqrt{2}L (m+1)^{\gamma^+}\,\beta_{T/(m+L)}(\delta) ~ \sqrt{T ~ \ln \frac{|V_{T/(m+L)}|}{|\lambda I_d|}}\nonumber\\
&\le 4L(m+1)^{\gamma^+} ~ \sqrt{T d~ \ln \left(1+\frac{T (m+1)^{2\gamma^+}}{d\lambda}\right)}\nonumber\\
&\hspace{2cm}\cdot\left(\sqrt{\lambda} + \sqrt{\ln\left(\frac{1}{\delta}\right) + d \, \ln\left(1 + \frac{T (m+1)^{2\gamma^+}}{d (m+L)\lambda}\right)}~\right)\,,\label{eq:intermediate_regret}
\end{align}
where we have used \eqref{eq:indiv_beta}, \eqref{eq:new_cs}, and \eqref{eq:bound_delayed_V}.
Similarly to \Cref{prop:adaptation_hp}, note that in the stationary case, i.e., when $m=0$ and $L=1$, we exactly recover \citep[Theorem~3]{abbasi2011improved}.
The first claim of \Cref{thm:regret_exp} is obtained by setting $\lambda \in [1, d]$, and $\delta = 1/T$.

Let $R_T$ denote the right-hand side of \eqref{eq:intermediate_regret}.
Combining this bound with the arguments of \Cref{prop:approx_cyclic}, we have with probability $1-\delta$
\begin{align}
\sum_{t=1}^T r_t &\ge \sum_{\tau=1}^{T/(m+L)} \tilde{r}(\ba_\tau) - \frac{m(m+1)^{\gamma^+}}{m+L}~T\label{eq:bnd_rwd_1}\\
&= \sum_{\tau=1}^{T/(m+L)} \langle \bb_\tau, \btheta^*\rangle - \frac{m(m+1)^{\gamma^+}}{m+L}~T\nonumber\\
&\ge \sum_{\tau=1}^{T/(m+L)} \langle \tilde{\bb}, \btheta^*\rangle - R_T - \frac{m(m+1)^{\gamma^+}}{m+L}~T\label{eq:use_regret}\\
&= \sum_{\tau=1}^{T/(m+L)} \tilde{r}(\tilde{\ba}) - R_T - \frac{m(m+1)^{\gamma^+}}{m+L}~T\nonumber\\
&\ge \sum_{t=1}^T \tilde{r}_t - R_T - \frac{2m(m+1)^{\gamma^+}}{m+L}~T\label{eq:bnd_rwd_2}\\
&\ge \OPT - R_T - \frac{4m(m+1)^{\gamma^+}}{m+L}~T\label{eq:use_approx}\\
&\ge \OPT - 4(m+1)^{\gamma^+} \Bigg[ \frac{mT}{m+L} + (m+L)\sqrt{T d~ \ln \left(1+\frac{T (m+1)^{2\gamma^+}}{d\lambda}\right)}\nonumber\\
&\hspace{5.15cm} \cdot\left(\sqrt{\lambda} + \sqrt{\ln\left(\frac{1}{\delta}\right) + d\ln\left(1 + \frac{T (m+1)^{2\gamma^+}}{d (m+L)\lambda}\right)}~\right)\Bigg]\,,\nonumber
\end{align}
where \eqref{eq:bnd_rwd_1} and \eqref{eq:bnd_rwd_2} come from the fact that any instantaneous reward is bounded by $(m+1)^{\gamma^+}$, see \eqref{eq:bound_R}, \eqref{eq:use_regret} from \eqref{eq:intermediate_regret}, and \eqref{eq:use_approx} from \Cref{prop:approx_cyclic}.

Now, assume that $m \ge 1$, $T \ge d^2m^2 + 1$, and let $L = \big\lceil \sqrt{m/d}~T^{1/4} \big\rceil - m$.
By the condition on $T$, we have $\sqrt{m/d}~T^{1/4} > m \ge 1$, such that $L \ge 1$ and
\[
\sqrt{\frac{m}{d}}\,T^{1/4} \le \left\lceil \sqrt{\frac{m}{d}}\,T^{1/4} \right\rceil = L+m \le \sqrt{\frac{m}{d}}\,T^{1/4} + 1 \le 2\sqrt{\frac{m}{d}}\,T^{1/4}\,.
\]
Substituting in the above bound, we have with probability $1-\delta$
\begin{align*}
\OPT - \sum_{t=1}^T r_t &\le 4\sqrt{d}\,(m+1)^{\frac{1}{2}+\gamma^+}~T^{3/4}\Bigg[1 + 2\sqrt{\ln \left(1+\frac{T (m+1)^{2\gamma^+}}{d\lambda}\right)}\\
&\hspace{4cm}\cdot\left(\sqrt{\frac{\lambda}{d}} + \sqrt{\frac{\ln(1/\delta)}{d} + \ln\left(1 + \frac{T (m+1)^{2\gamma^+}}{d\lambda} \right)}~\right)\Bigg]\,.
\end{align*}
The second claim of \Cref{thm:regret_exp} is obtained by setting $\lambda \in [1, d]$, and $\delta=1/T$.
\end{proof}

\subsection{Proof of Corollary \ref{cor:model_selection}}
\label{apx:corollary}

First, we state a lemma which links the normalized regret of a block meta-algorithm to the true regret of the corresponding sequence of blocks.

\begin{lemma}\label{lem:average}
Suppose that a block-based bandit algorithm (in our case the bandit combiner) produces a sequence of $T_\textnormal{bc}$ blocks $\ba_\tau$, with possibly different cardinalities $|\ba_\tau|$, such that
\[
\sum_{\tau=1}^{T_\textnormal{bc}} \frac{\tilde{r}(\tilde{\ba})}{|\tilde{\ba}|} - \sum_{\tau=1}^{T_\textnormal{bc}} \frac{\tilde{r}(\ba_\tau)}{|\ba_\tau|} \le F(T_\textnormal{bc})\,,
\]
 for some sublinear function $F$. Then, we have
\[
\frac{\min_\tau |\ba_\tau|}{\max_\tau |\ba_\tau|}\left(\tilde{r}(\tilde{\ba}) ~ \frac{\sum_\tau |\ba_\tau|}{|\tilde{\ba}|}\right) - \sum_{\tau=1}^{T_\textnormal{bc}} \tilde{r}(\ba_\tau) ~\le~ \min_\tau |\ba_\tau|\,F(T_\textnormal{bc})~.
\]
In particular, if all blocks have same cardinality the last bound is just the block regret bound scaled by $|\ba_\tau|$.
\end{lemma}

\begin{proof}
We have
\begin{align*}
\sum_{\tau=1}^{T_\text{bc}} \tilde{r}(\ba_\tau) & \ge \min_\tau |\ba_\tau| ~ \sum_{\tau=1}^{T_\text{bc}} \frac{\tilde{r}(\ba_\tau)}{|\ba_\tau|}\\
&\ge \min_\tau |\ba_\tau| \left(\sum_{\tau=1}^{T_\text{bc}} \frac{\tilde{r}(\tilde{\ba})}{|\tilde{\ba}|} - F(T_\text{bc})\right)\\
&= \frac{\min_\tau |\ba_\tau|}{\max_\tau |\ba_\tau|}~\frac{\tilde{r}(\tilde{\ba})}{|\tilde{\ba}|}~\max_\tau |\ba_\tau|~T_\text{bc} - \min_\tau |\ba_\tau|\,F(T_\text{bc})\\
&\ge \frac{\min_\tau |\ba_\tau|}{\max_\tau |\ba_\tau|}\left(\tilde{r}(\tilde{\ba})~\frac{\sum_\tau |\ba_\tau|}{|\tilde{\ba}|}\right) - \min_\tau |\ba_\tau|\,F(T_\text{bc})\,.
\end{align*}
\end{proof}

\cormodelselection*

\begin{proof}
%
%
Let $m_\star$ be the true memory size, and $L_\star = L(m_\star)$ the corresponding (partial) block length.
Throughout the proof, $\tilde{\ba}$ denotes the block defined in \eqref{eq:best_block} with length $m_\star + L_\star$.
First observe that only one of the OFUL-memory instances we test is well-specified, i.e., has the true parameters $(m_\star, \gamma_\star)$.
We can thus rewrite the regret bound for the Bandit Combiner \citep[Corollary~2]{cutkosky2020upper}, generalized to rewards bounded in $[-R, R]$ as follows
\begin{equation}\label{eq:reg-bc-extended}
\text{Regret}_\text{bc} \le \tilde{\mathcal{O}}\left(C_\star T_\text{bc}^{\alpha_\star} + C_\star^{\frac{1}{\alpha_\star}} T_\text{bc} \eta_{\star}^{\frac{1 - \alpha_\star}{\alpha_\star}} + R^2 T_\text{bc} \eta_{\star} + \sum_{j \ne \star} \frac{1}{\eta_j}\right)\,,
\end{equation}
where $T_\text{bc} = T/(m_\star + L_\star)$ is the bandit combiner horizon, $C_\star$ and $\alpha_\star$ are the constants in the regret bound of the well-specified instance (see below how we determine them), and the $\eta_j$ are free parameters to be tuned.
We now derive $C_\star$ and $\alpha_\star$.
To that end, we must establish the regret bound of the well-specified instance, and identify $C_\star$ and $\alpha_\star$ such that this bound is equal to $C_\star T_\textnormal{bc}^{\alpha_\star}$, where $C_\star$ may contain logarithmic factors. 
For the well-specified instance, the first claim of \Cref{thm:regret_hp} gives that, with probability at least $1-\delta$, we have
\begin{align}
\sum_{\tau = 1}^{T/(m_\star+L_\star)} \br(\tilde{\ba}) - \br(\ba_\tau) &\le 4(m_\star+L_\star)  (m_\star + 1)^{\gamma_{\star}^+} ~ \sqrt{T d~ \ln \left(1+\frac{T (m_\star + 1)^{2 \gamma_{\star}^+} }{d\lambda}\right)}\nonumber\\
&\hspace{2.5cm}\left(\sqrt{\lambda} + \sqrt{\ln\left(\frac{1}{\delta}\right) + d \, \ln\left(1 + \frac{T (m_\star + 1)^{2 \gamma_{\star}^+}}{d(m_\star + L_\star)\lambda}\right)}~\right)\nonumber\\[0.4cm]
\sum_{\tau = 1}^{T/(m_\star+L_\star)} \frac{\br(\tilde{\ba})}{|\tilde{\ba}|} - \frac{\br(\ba_\tau)}{|\ba_\tau|} &\le T^{1/2}\, 4 (m_\star + 1)^{\gamma_{\star}^+} ~ \sqrt{d \ln \left(1+\frac{T(m_\star + 1)^{2 \gamma_{\star}^+}}{d\lambda}\right)}\label{eq:regret_wrong_horizon}\\
&\hspace{2.3cm}\left(\sqrt{\lambda} + \sqrt{\ln\left(\frac{1}{\delta}\right) + d \, \ln\left(1 + \frac{T (m_\star + 1)^{2 \gamma_{\star}^+}}{d(m_\star + L_\star)\lambda}\right)}~\right)\,,\nonumber
\end{align}
where we have used that $|\ba_\tau| = |\tilde{\ba}| = m_\star + L_\star$ for every $\tau$.
Note that the right-hand side of \eqref{eq:regret_wrong_horizon} is expressed in terms of $T$, which is not the correct horizon, $T/(m_\star + L_\star)$.
However, recall that we have
\begin{align*}
m_\star + L_\star &\le 2 \sqrt{\frac{m_\star}{d}}\,T^{1/4}\\
(m_\star + L_\star)^4 &\le \left(\frac{4 m_\star}{d}\right)^2 T\\
T^3 &\le \left(\frac{4 m_\star}{d}\right)^2 \left(\frac{T}{m_\star + L_\star}\right)^4\\
T^{1/2} &\le \left(\frac{4 m_\star}{d}\right)^{1/3} \left(\frac{T}{m_\star + L_\star}\right)^{2/3}\,,
\end{align*}
such that by substituting in \eqref{eq:regret_wrong_horizon} and identifying we have $\alpha_\star = 2/3$, and
\begin{align*}
C_\star &= 4\left(\frac{4 m_\star}{d}\right)^{1/3}(m_\star + 1)^{\gamma_\star^+} ~ \sqrt{d \ln \left(1+\frac{T_\text{bc} (m_\star + L_\star) (m_\star + 1)^{2 \gamma_{\star}^+}}{d\lambda}\right)}\\
&\hspace{5cm}\left(\sqrt{\lambda} + \sqrt{\ln\left(\frac{1}{\delta}\right) + d \, \ln\left(1 + \frac{T_\text{bc} (m_\star + 1)^{2 \gamma_\star^+}}{d \lambda}\right)}~\right)\,.
\end{align*}

Setting $\eta_j = T_\text{bc}^{-2/3}$, and substituting in \eqref{eq:reg-bc-extended} with $R = (m_{\star} + 1)^{\gamma_{\star}^+}$, we have that with high probability
\[
\sum_{\tau=1}^{T_\text{bc}} \frac{\br(\tilde{\ba})}{|\tilde{\ba}|} - \frac{\br(\ba^\text{bc}_\tau)}{|\ba^\text{bc}_\tau|} = \tilde{\mathcal{O}}\Big( \big(C_\star^{3/2} + N\big)\,T_\text{bc}^{2/3} + \ (m_{\star} + 1)^{2 \gamma_{\star}^+}\,T_\text{bc}^{1/3}\Big)\,.
\]

Now, recall that $T_\text{bc} = \mathcal{O}\big(\sqrt{d/m_\star}\,T^{3/4}\big)$, and that $C_\star = \tilde{\mathcal{O}}\big((m_\star + 1)^{\frac{1}{3} + \gamma_{\star}^+} \, d^{2/3}\big)$.
Hence, $N \le d \sqrt{m_\star}$ implies $N = \mathcal{O}\big(C_j^{3/2}\big)$, and $(m_\star + 1)^{\gamma_{\star}^+} \le d^2 \sqrt{m_\star T}$ implies $(m_\star + 1)^{\gamma_{\star}^+}\,T_\textnormal{bc}^{1/3} = \mathcal{O}\big(C_\star^{3/2}\, T_\textnormal{bc}^{2/3}\big)$.
Setting $\lambda \in [1, d]$, $\delta = 1/T$, we obtain
\begin{equation}\label{eq:O_tilde}
\mathbb{E}\left[\sum_{\tau=1}^{T_\text{bc}} \frac{\br(\tilde{\ba})}{|\tilde{\ba}|} - \frac{\br(\ba^\text{bc}_\tau)}{|\ba^\text{bc}_\tau|}\right] = \tilde{\mathcal{O}}\Big( d\,\sqrt{m_\star}\,(m_\star + 1)^{\frac{3}{2} \gamma_{\star}^+}\,T_\text{bc}^{2/3}\Big)\,.
\end{equation}

Let $m_\tau$ be the memory size associated to the bandit played at block time step $\tau$ by \Cref{alg:bc}. Let $m_\text{min} = \min_j m_j$ and $m_\text{max}= \max_j m_j$. Finally, let $L_\text{min}$ and $L_\text{max}$ the (partial) block length associated with $m_\text{min}$ and $m_\text{max}$.
We have
\[
\sum_{t=1}^T r^\text{bc}_t \ge \sum_{\tau=1}^{T_\text{bc}} \left(\tilde{r}(\ba^\text{bc}_\tau) - m_\tau\,(m_\star+ 1)^{\gamma_{\star}^+} \right) \ge\sum_{\tau=1}^{T_\text{bc}} \tilde{r}(\ba^\text{bc}_\tau) - m_\text{max}\,(m_{\star} + 1)^{\gamma_{\star}^+}\,T_\text{bc}\,,
\]
such that by \Cref{lem:average} and \eqref{eq:O_tilde} we obtain
\begin{align*}
&\mathbb{E}\left[\frac{\min_\tau |\ba_\tau|}{\max_\tau |\ba_\tau|}\left(\tilde{r}(\tilde{\ba}) ~ \frac{\sum_\tau |\ba_\tau|}{|\tilde{\ba}|}\right) - \sum_{t=1}^T r^\text{bc}_t\right]\\
&\hspace{3.5cm}\le m_\text{max}\,(m_{\star} + 1)^{\gamma_{\star}^+}\,T_\text{bc} + \min_\tau |\ba_\tau|\,\tilde{\mathcal{O}}\Big( d\,\sqrt{m_\star}\,(m_\star + 1)^{\frac{3}{2} \gamma_{\star}^+}\,T_\text{bc}^{2/3}\Big)\,,
\end{align*}
\begin{align*}
&\mathbb{E}\left[\frac{m_\text{min} + L_\text{min}}{m_\text{max} + L_\text{max}}\left( \frac{L_\star\,\OPT}{T} ~ \frac{T}{m_\star + L_\star}\right) - \sum_{t=1}^T r^\text{bc}_t\right]\\
&\hspace{2.9cm}\le \frac{m_\text{max}\,(m_{\star} + 1)^{\gamma_{\star}^+}\,T}{m_\text{min} + L_\text{min}} + (m_\text{min} + L_\text{min})^{1/3}\,\tilde{\mathcal{O}}\Big(d\,\sqrt{m_\star}\,(m_\star + 1)^{\frac{3}{2} \gamma_{\star}^+}\,T^{2/3}\Big)\,,
\end{align*}

\begin{align*}
\mathbb{E}\left[\sqrt{\frac{m_\text{min}}{m_\text{max}}}~ \OPT - \sum_{t=1}^T r^\text{bc}_t\right] &\le \frac{m_\text{max}}{m_\text{min}} \sqrt{d\,m_\star}\,(m_\star + 1)^{\gamma_{\star}^+}\,T^{3/4} + \tilde{\mathcal{O}}\Big(d\,m_\star\,(m_\star + 1)^{ \frac{3}{2} \gamma_{\star}^+}\,T^{3/4}\Big)\\[0.2cm]
&= \frac{m_\text{max}}{m_\text{min}} ~ \tilde{\mathcal{O}}\Big(d\,m_\star\, (m_\star + 1)^{ \frac{3}{2} \gamma_{\star}^+}\,T^{3/4}\Big)\,,
\end{align*}
where we have used the fact that $m_\text{min} + L_\text{min} = \sqrt{m_\text{min}/d}~T^{1/4}$, and $m_\text{max} + L_\text{max} = \sqrt{m_\text{max}/d}~T^{1/4}$.
\Cref{cor:model_selection} is obtained by setting $M = m_\text{max}/m_\text{min}$.
\end{proof}


\section{Bandit Combiner}
\label{apx:bc}

In this section we show our adaptation of the Bandit Combiner algorithm \cite{cutkosky2020upper} to instances of \texttt{O3M}.
Recall that numbers $C_j$ and target regrets $R_j$ for $\texttt{O3M}(m_j, \gamma_j)$, $j=1, \dots, N$, are defined as
\begin{align}
C_j &= 4\left(\frac{4 m_j}{d}\right)^{1/3}(m_j + 1)^{\gamma_j^+} ~ \sqrt{d \ln \left(1+\frac{T_\text{bc} (m_j + L_j) (m_j + 1)^{2 \gamma_{j}^+}}{d\lambda}\right)}\label{eq:C_j}\\
&\hspace{5cm}\left(\sqrt{\lambda} + \sqrt{\ln\left(\frac{1}{\delta}\right) + d \, \ln\left(1 + \frac{T_\text{bc} (m_j + 1)^{2 \gamma_j^+}}{d \lambda}\right)}~\right)\,,\nonumber
\end{align}
\begin{align*}
R_j &= C_j T_\text{bc}^{\alpha_j} + \frac{(1 - \alpha_j)^{\frac{1-\alpha_j}{\alpha_j}} (1 + \alpha_j)^{\frac{1}{\alpha_j}}}{\alpha_j^{\frac{1 - \alpha_j}{\alpha_j}}} C_j^{\frac{1}{\alpha_j}} T_\text{bc} \eta_j^{\frac{1-\alpha_j}{\alpha_j}}\\
&\hspace{1.5cm}+ 1152 (m_j + 1)^{2 \gamma_j^+} \log(T_\text{bc}^3 N / \delta) T_\text{bc} \eta_j + \sum_{k \neq j} \frac{1}{\eta_k}.
\end{align*}
Note that the form of the target regret $R_j$ slightly differs from the one presented in \cite[Corollary~2]{cutkosky2020upper} due to the different range of the rewards.
With our choices for $C_j$, which defined as \eqref{eq:C_j}, $\alpha_j = 2/3$, and $\eta_j = 1/T_\text{bc}^{2/3}$ for $j=1, \dots, N$, the target regrets become 
\begin{equation}\label{eq:regret_cutko}
    R_j = C_j\,T_\text{bc}^{2/3} + \frac{5\sqrt{30}}{18} 
    C_j^{3/2}\,T_\text{bc}^{2/3} + 1152 (m_j + 1)^{2 \gamma_j^+}\,T^{1/3} \log(T_\text{bc}^3N/\delta) + (N-1) T^{2/3}\,.
\end{equation}
where we note how that the presence of $(m_j + 1)^{2 \gamma_j^+}$ is impacting differently the rising and rotting scenarios.
The algorithm, which is an adaptation of Bandit Combiner in \cite{cutkosky2020upper}, is summarized in \Cref{alg:bc}.

\begin{algorithm}[!ht]
\SetKwInOut{Input}{input}
\SetKwInOut{Init}{init}
\SetKwInOut{Parameter}{Param}
\caption{\texttt{Bandit Combiner on Over-Optimistic OFUL-Memory (\texttt{O3M})}}
\Input{Instances $\texttt{O3M}(m_1, \gamma_1), \ldots, \texttt{O3M}(m_N, \gamma_N)$, horizon $T_\text{bc}$ \\
numbers $C_1, \dots, C_N > 0$, 
target regrets $R_1, \dots, R_N$.\\\vspace{0.15cm}
Set $T(i) = 0, \mathcal{S}_i = 0, \Delta_i = 0$ for $i=1, \dots, N$, and set $I_0 = \{1, \dots, N\}$}\vspace{0.15cm}
\For{ $t=1, \dots, T_\text{bc}$ }{\vspace{0.15cm}
    \uIf{there is some $i \in I_t$ with $T(i)=0$}{$i_t=i$}
    \uElse{
    For each $i \in I_t$, compute the UCB index:
    \begin{align*}
        \mathrm{UCB}(i) = \min\left\{ (m_i + 1)^{2 \gamma_i^+}, \frac{C_i}{\sqrt{T(i)}} + 4 (m_i + 1)^{2 \gamma_i^+}\sqrt{\frac{2\log(T^3 N / \delta)}{T(i) }}\right\} - \frac{R_i}{T_\text{bc}}
    \end{align*}
    Set $i_t = \argmax_{i \in I_t} \frac{\mathcal{S}_{i}}{T(i)} + \mathrm{UCB}(i)$ }
    \hspace{0.005cm}Obtain from instance $\texttt{O3M}(m_{i_t}, \gamma_{i_t})$ a block of size $m_{i_t} + L_{i_t}$ and play it\vspace{0.15cm}\\
    Return the total reward $r_{i_t}$ collected in the last $L_{i_t}$ time steps of the block to $\texttt{O3M}(m_{i_t}, \gamma_{i_t})$ \vspace{0.15cm}\\
    Compute the average reward $\widehat{r}_{i_t} = \frac{r_{i_t}}{L_{i_t}}$
    \vspace{0.15cm}\\
    Update $\Delta_{i_t} \leftarrow \Delta_{i_t} + \mathcal{S}_{i_t}/T(i_t) - \widehat{r}_{i_t}$ (where we set $0/0 = 0$) and $\mathcal{S}_{i_t} \leftarrow \mathcal{S}_{i_t} + \widehat{r}_{i_t}$
    \vspace{0.15cm}\\
    Update the number of plays $T(i_t) \leftarrow T(i_t) + 1$\vspace{0.15cm}\\
    \uIf{
$\Delta_{i_t}
\ge C_{i_t} T(i_t)^{\gamma_{i_t}} + 12 \ (m_{i_t} + 1)^{2 \gamma_{i_t}^+} \sqrt{2 \log(T^3 N /\delta) T(i_t)}$}{\vspace{0.2cm}{
    $I_t = I_{t-1} \setminus \{i_t\}$
    }}
\uElse{
$I_t = I_{t-1}$
}
}
\label{alg:bc}
\end{algorithm}

%



\section{Additional Experiments}
\label{sec:more-exp}
We provide an additional experiment comparing the regrets of \texttt{O3M} and \texttt{OM-Block}. In order to be able to plot the regret, we must know OPT which is hard to compute in general. Since in the rising scenario with an isotropic initialization OPT is oracle greedy, which is easy to compute, we present this experiment in a rising setting with $m=1$ and $\gamma=2$. We plot the regret of \texttt{O3M} and \texttt{OM-Block} against the number of time steps, measuring the performance at different time horizons and for different sizes of $L$ (where $L$ depends on $T$, see at the end of \Cref{sec:estimation}). Specifically, we instantiated \texttt{O3M} and \texttt{OM-Block} for increasing values of $L$, setting the horizon of each instance based on the equations in \Cref{thm:regret_exp} and \Cref{prop:adaptation_exp}. \Cref{fig:regret} shows how the dimension of $\hat{\theta}$, which is $d$ for \texttt{O3M} and $d \times L$ for \texttt{OM-Block}, has an actual impact on the performance since \texttt{O3M} outperforms \texttt{OM-Block}. 

The code is written in Python and it is publicly available at the following GitHub repository: \href{https://anonymous.4open.science/r/Linear-Bandits-with-Memory-4B37/}{Linear Bandits with Memory}.

\begin{figure*}[!b]
    \centering
    \includegraphics[width=.5\textwidth]{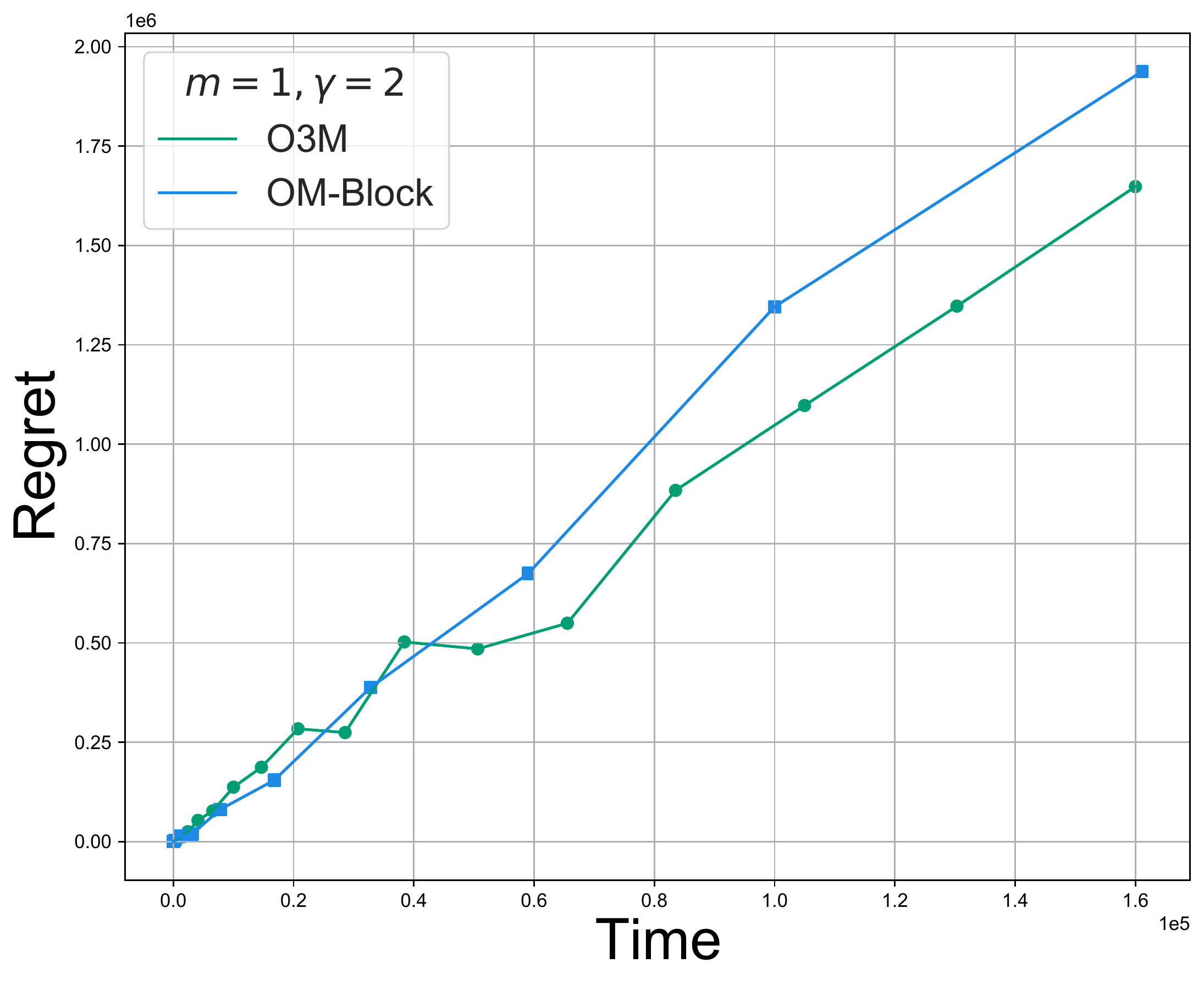}
    \caption{The regret of \texttt{O3M} and \texttt{OM-Block}. Each dot is a separate run where the value of $L$ is tuned to the corresponding horizon.}
    \label{fig:regret}
\end{figure*}



\end{document}